\documentclass[a4paper,12pt]{article}

\usepackage{booktabs} % For formal tables

\usepackage[utf8]{inputenc}
\usepackage[T1]{fontenc}
\usepackage{mathtools, amsthm, amsmath,amsfonts, dsfont, version}
\usepackage[linesnumbered,ruled]{algorithm2e}
\usepackage{colortbl}
\usepackage[normalem]{ulem}

\newcommand{\E}{\{0,1\}^n}
\newcommand{\N}{\mathbb{N}}

\newcommand{\R}{\mathbb{R}}
\newcommand{\M}{\mathcal{M}}

\newcommand{\eps}{\varepsilon}

\newcommand{\oea}{$(1 + 1)$~EA\xspace}

\newcommand{\ea}{$(\mu, \lambda)$~EA\xspace}
\newcommand{\oclea}{$(1, \lambda)$~EA\xspace}
\newcommand{\ml}{$(\mu + \lambda)$\xspace}

\newcommand{\onemax}{\textsc{OneMax}\xspace}

\DeclarePairedDelimiter\ceil{\lceil}{\rceil}

\newcommand{\abs}[1]{\left\lvert#1\right\rvert}

\DeclareMathOperator{\Bin}{Bin}
\DeclareMathOperator{\Ber}{Ber}
\DeclareMathOperator{\Var}{Var}

\newtheorem{theorem}{Theorem}
\newtheorem{lemma}{Lemma}
\newtheorem{definition}{Definition}
\newtheorem{corollary}{Corollary}

\begin{document}
\title{The Efficiency Threshold for the Offspring Population Size of the ($\mu$, $\lambda$) EA}
% \titlenote{Produces the permission block, and
%   copyright information}
% \subtitle{Subtitle}
% \subtitlenote{The full version of the author's guide is available as
%   \texttt{acmart.pdf} document}

%%% The submitted version for review should be ANONYMOUS
\author{Denis Antipov\\
ITMO University \\
St. Petersburg \\
Russia \\
\and
Benjamin Doerr\\
\raisebox{0mm}[0mm][0mm]{\'E}cole Polytechnique \\
 CNRS \\
Laboratoire d'Informatique (LIX)\\
Palaiseau\\
France \\
\and
Quentin Yang \\
\raisebox{0mm}[0mm][0mm]{\'E}cole Polytechnique \\
Palaiseau\\
France}

\maketitle

\begin{abstract}
  Understanding when evolutionary algorithms are efficient or not, and how they efficiently solve problems, is one of the central research tasks in evolutionary computation. In this work, we make progress in understanding the interplay between parent and offspring population size of the $(\mu,\lambda)$ EA. Previous works, roughly speaking, indicate that for $\lambda \ge (1+\eps) e \mu$, this EA easily optimizes the OneMax function, whereas an offspring population size $\lambda \le (1 -\eps) e \mu$ leads to an exponential runtime.

  Motivated also by the observation that in the efficient regime the $(\mu,\lambda)$ EA loses its ability to escape local optima,  we take a closer look into this phase transition. Among other results, we show that when $\mu \le n^{1/2 - c}$ for any constant $c > 0$, then for any $\lambda \le e \mu$ we have a super-polynomial runtime. However, if $\mu \ge n^{2/3 + c}$, then for any $\lambda \ge e \mu$, the runtime is polynomial. For the latter result we observe that the $(\mu,\lambda)$ EA profits from better individuals also because these, by creating slightly worse offspring, stabilize slightly sub-optimal sub-populations. While these first results close to the phase transition do not yet give a complete picture, they indicate that the boundary between efficient and super-polynomial is not just the line $\lambda = e \mu$, and that the reasons for efficiency or not are more complex than what was known so far.
\end{abstract}

\section{Introduction}

While the theory of evolutionary algorithms (EAs) has made considerable progress in the last 20 years, several topics remain little understood and pose problems to a rigorous analysis, among them non-trivial populations and non-elitist algorithms. As examples, we note that the asymptotically precise runtime of the \ml on the \onemax benchmark function was only determined very recently~\cite{AntipovDFH18}, whereas the runtime of the \ea on this simple function is not yet determined asymptotically precise. Consequently, we do not fully comprehend the working principles of populations and comma selection. To try to overcome this shortage, we continue the classic line of theoretical research of regarding simple test functions, gaining a rigorous understanding how simple EAs optimize these, and from this try to gain a broader understanding of certain working principles. In short, in this work we continue the existing research efforts of understanding how the \ea optimizes the \onemax function, though with different methods and with a higher degree of precision than before.

What is known about how the \ea optimizes \onemax is roughly the following. When the offspring population size $\lambda$ is at most $(1-\eps)e\mu$ for some positive constant $\eps$, then the expected runtime (measured by the number of fitness evaluations until an optimum is found) is exponential in~$n$~\cite{Lehre10}. When $\lambda \ge (1+\eps)e\mu$ and $\lambda \ge C \ln n$ with $C$ a sufficiently large constant, then the runtime becomes polynomial, and in fact, $O(n \lambda \log \lambda)$~\cite{DangL16algo}.

There is a good reason for these results. Let $x$ be a parent individual with high fitness, that is, $\onemax(x)$ is close to $n$ and thus $d := d(x) := n - \onemax(x)$ is small. When generating offspring from $x$ via standard-bit mutation with mutation rate $\frac 1n$, then with probability roughly $\frac 1 e$ the offspring has the same fitness as the parent, with probability $O(\frac dn)$ the offspring is better than the parent, and else it is worse. Consequently, when $d$ is small, the number of individuals with best fitness in the population, in expectation, increases per iteration by a factor of at least $(1+\eps)$ when $\lambda \ge (1+\eps)e\mu$ and it decreases by a factor of roughly $(1-\eps)$ when $\lambda \le (1-\eps)e\mu$. In the efficient case, the $(1+\eps)$ multiplicative increase of the number of top-individuals suffices to ensure that a single top individual has a constant chance to take over the whole population in $O(\log \mu)$ iterations. We note that a number of highly non-trivial arguments~\cite{Lehre10,DangL16algo} are necessary to transform these observations into rigorous proofs for the runtimes cited above.

While these results are mathematically non-trivial despite their intuitive explanations, they only discuss the easy situations where the number of top individuals is subject to a clear drift, either into the right or the wrong direction. These situations might be too extreme to lead to a full understanding of the population dynamics of this EA. Moreover, these are typically the situations in which using comma selection is not a good idea. For the case of negative, but also positive drift, that is, $\lambda \ge (1+\eps)e\mu$ and $\lambda$ sufficiently large, the main advantage of comma selection is absent. We recall that comma selection is used, among others, with the hope that by not keeping good parent individuals in the population, one can prevent premature convergence. If $\lambda \ge (1+\eps)e\mu$ and $\lambda$ is sufficiently large, then discarding the parent population does not help, since with high probability it reappears in the offspring population.

To be more precise, let us assume that we have a parent population that is converged to a local optimum. Then with $\lambda \ge (1+\eps)e\mu$, an expected number of $(1+\eps) \mu$ copies of this parent are generated as offspring. Since these are generated independently, with probability $1 - e^{-\Omega(\mu)}$ at least $\mu$ such copies are generated, which means that inferior offspring cannot enter the population. For this reason, the two regimes with clear drift are possibly not the most interesting ones for using an EA with comma selection.

\textbf{Our results:} To gain a deeper understanding of the population dynamics of EAs in the case where there is no clear drift, we regard settings with $\lambda$ closer to $e\mu$. We prove three results inside this phase transition region $\lambda = (1 \pm \eps) e \mu$.

We first show that the super-polynomial range already starts when $\lambda \le (1-\eps)e\mu$ for $\eps = \omega(n^{-1/2})$. To prove this result, we do not extend the general but technical negative-drift-in-populations theorem of~\cite{Lehre10} to smaller negative drifts, but instead use a basic drift argument. This approach avoids the use of family trees and branching processes and might thus be a light-weight alternative for similar analysis problems as well.

When $\mu$ is not overly large, namely $\mu \le n^{1/2-c}$ for an arbitrary small constant $c > 0$, then the weaker condition $\lambda \le e \mu$ suffices to lead to a super-polynomial runtime. Note that in this regime, we have essentially no drift in the sub-population of best individuals. The reason why the \ea still has difficulties to find the optimum is that in this no-drift regime, the number of best individuals performs an unbiased a random walk (with typical step sizes up to $\sqrt \mu$). When this walk reaches zero, no individual on this fitness level is left and the \ea, due to the limited population size, takes a non-trivial amount of time to re-generate such an individual. The time this walk takes to reach zero is roughly $O(\mu)$. Hence if $\mu \le n^{1/2-c}$ and the best fitness in the population is close to $n$, then the $O(\mu)$ iterations with $O(\lambda) = O(\mu)$ offspring generated are not enough to produce a strictly better individual. For this reason, we exhibit here a (slow, namely constant per $O(\mu)$ iterations) negative drift in the fitness of the best individual in the population. This negative drift translates into a long runtime via a negative drift theorem due to Hajek~\cite{Hajek82}.

When $\mu$ is slightly larger, namely at least $n^{2/3 + c}$ for an arbitrary constant $c>0$, then for all $\lambda \ge e \mu$, that is, again including settings with essentially no drift, we have a polynomial runtime of $O(n \lambda \log n)$, which means  $O(n \log n)$ iterations. This is, the same runtime guarantee as shown for the constant $(1+\eps)$ drift case in~\cite{DangL16algo}, but the reasons are different. Here, we have essentially a no-drift regime. Hence the number of individuals on the highest fitness level performs an unbiased random walk. Different from above, the larger population sizes implies that before this walk reaches zero, some individuals are generated on a higher level. This is not the immediate pathway to the optimum since these small sub-populations have a good chance of dying out quickly (they perform the same type of unbiased random walk, but starting close to zero). The reason why these climbers make a difference is that they stabilize the fitness level below them. We recall that such an individual, when chosen as parent, creates an equally fit offspring with probability roughly $\frac 1e$ . In addition, with probability $\frac 1e$ it creates an offspring on the next lower fitness level. These offspring create a positive drift in this level and hinder it from dying out after $O(\mu)$ iterations. Consequently, this lower level has ample time to create further climbers until one of them successfully take over the population.

\subsection*{Related Work}

For reasons of space, we shall not discuss the full literature on theoretical works on population-based and non-elitist EAs. We refer to the textbooks~\cite{AugerD11,Jansen13,NeumannW10} for a good overview of the field. Clearly visible is that the vast majority of the works in this field considers elitist EAs, and often, the \oea with trivial populations, whereas non-elitism appears only in a small number of works which, e.g., discuss the influence of different selection mechanisms. So we mention only two strongly related series of works.

The very general analyses of non-elitist EAs in~\cite{Lehre10,Lehre11,DangL16algo,CorusDEL18} give as special case the results for the \ea mentioned above. The downside of such a general machinery is that it gives the non-expert less understanding of how the \ea really solves a problem. This is particularly true for the general results for upper bounds~\cite{Lehre11,DangL16algo,CorusDEL18}, which are proven via an intricate potential function argument. Consequently, our insight that in a run of the \ea with $\lambda = (1+\eps) e \mu$, the best individual with constant probability takes over the whole population in $O(\log \mu)$ iterations is not easily derived from these works.

A threshold behavior for the \oclea was observed in~\cite{JagerskupperS07,NeumannOW09,RoweS14}. The latest of these works~\cite{RoweS14} shows that for $\lambda \ge \log_{\frac{e}{e-1}} n \approx 2.18 \ln n$ the \oclea optimizes \onemax in an expected number of $O(n \log n + \lambda n)$ fitness evaluations, whereas for $\lambda \le (1-\eps) \log_{\frac{e}{e-1}} n$, the runtime is $\exp(\Omega(n^{\eps/2}))$ with high probability.

\section{Preliminaries and Notation}

\subsection{Notation}
By the set of natural integers $\N$ we denote the set of non-negative integers $\{0, 1, 2,\dots\}$.

For any probability distribution $\mathcal{L}$ and random variable $X$, we write $X\sim\mathcal{L}$ to indicate that $X$ follows the law $\mathcal{L}$.
We denote the Bernoulli law of parameter $p\in[0,1]$ by $\Ber(p)$
and the binomial law with parameters $m\in\N$ and $p\in[0,1]$ by $\Bin(m,p)$.

An empty product (i.e. a product over an empty set) is always considered to be $1$, an empty sum is always $0$. The infimum of the empty set is $+\infty$.

% A vector in $\E$ is called an individual. A set of $\mu$ individuals is called a population.
\subsection{Problem Statement}

In this work we consider the optimization of $n$-dimensional pseudo-Boolean functions $\E \to \R$. In particular, we regard the \onemax function which returns the number of one-bits in its argument. %\da{To Benjamin. Add some words that despite that it looks innicent, analysis of the EAs was very fruitful?}
% The value of the \onemax function depends only on the Hamming distance from its argument to its optimum $x^* = 1^n$. For any $x \in \E$ w
We call the value $\onemax(x)$ \emph{the fitness of $x$} and for brevity we denote it by $f(x)$.

% A $n$-dimensional \onemax function is defined on $\E$ and is strictly monotonic with respect to the Hamming distance. Therefore, it has a unique optimum $x^*$. For convenience only, we assume that $x^*$ is the all-$1$ bit-string and we define the notion of fitness as follows. For all $x\in\E$, we define the \emph{fitness} of $x$ by $f(x)=\sum\limits_{i=0}^nx_i$. Intuitively, the fitness indicates how close $x$ is to $x^*$.  In this paper, we consider that we are able to compute the fitness and we define the \emph{runtime} of an algorithm by the number of fitness evaluations and we give some result about the expected runtime for the optimization of a $n$-dimensional \onemax function with the \ea.

We analyze the performance of the \ea when optimizing pseudo-Boolean functions such as \onemax.
The \ea is a non-elitist evolutionary algorithm. It starts with a population that consists of $\mu$ random vectors from $\E$. Then it repeats the following cycle until some stopping criteria is met. The algorithm chooses an individual $x$ from the population uniformly at random and then creates its offspring by copying $x$ and flipping each bit independently with probability $\frac{1}{n}$. After obtaining $\lambda$ offspring it chooses the $\mu$ best (in terms of fitness) of them as the next population. The pseudo-code of the \ea is shown in Algorithm~\ref{alg:ea}.

\begin{algorithm}[h]
    % \SetKwInOut{Input}{Input}
    % \SetKwInOut{Initialization}{Initialization}
    % \Input{$n\geq1$, target function $f$ over $\E$}
    Create a population $P_0$ of $\mu$ individuals by choosing each individual from $\E$	u.a.r.\;
    $t \gets 0$\;
    \While{not terminated}
        {
        \For{$i \in [1..\lambda]$}
            {$x_i \gets$ a copy of an individual chosen uniformly at random from the parent population $P_t$\;
            Flip each bit in $x_i$ independently with probability $\frac{1}{n}$\;
            }
        $t \gets t + 1$\;
        $P_t \gets \mu$ individuals with largest $f$-values among $x_1,\dots,x_\lambda$. Ties are broken randomly.
        }
    \caption{The \ea maximizing $f:\E\to\R$}
    \label{alg:ea}
\end{algorithm}

Every iteration of the outer loop is called a \emph{generation}. For $t\in\N$, we define $P_t$ as the parent population of the algorithm after generation $t$.
% We denote by $\mathcal{G}$ the generation operator over the space of populations. It takes an argument $P_t$ and computes $P_{t+1}=\mathcal{G}(P_t)$ the population obtained from $P_t$ after one generation. Similarly,
We denote by $\mathcal{M}$ the mutation operator over $\E$. It takes an argument $x\in\E$ and computes $\mathcal{M}(x)$ by flipping independently each bit of $x$ with a probability $\frac{1}{n}$.
For all individuals $x\in\E$, we define the difference in fitness
\[\delta_x:=f(\mathcal{M}x)-f(x).\]
Note that both $\M(x)$ and $\delta_x$ are random variables.

We call the \emph{runtime} of an optimization algorithm the number of evaluations of the target function until this algorithm finds an optimum. %Our goal is to show a precise threshold for parameters $\lambda$ and $\mu$ between polynomial and super-polynomial runtimes of the \ea.

\subsection{Useful Tools}

\textbf{Transition probabilities.}
We have the following two estimates for the distribution of $\delta_x$.
% mutation law lemma

\begin{lemma}
\label{lem:mutlaw}
Let $x$ be an individual of fitness $f(x) = n - d$. Then, for all $k \geq 1$, we have
\[\Pr(\delta_x=k)\leq \dbinom{d}{k}\left(\frac{1}{n}\right)^k.\]
\end{lemma}

% \begin{proof}
% If $k\geq1$, then to increase the fitness by exactly $k$ we need to flip at least $k$ bits among the $d$ wrong bits.
% \end{proof}

\begin{lemma}
\label{lem:Pdelta_x=0}
Let $x$ be an individual of fitness $f(x)=n-d$. Then
\[\Pr(\delta_x=0)=\sum\limits_{k=0}^{\min\{d, n - d\}}\dbinom{d}{k}\dbinom{n-d}{k}\left(\frac{1}{n}\right)^{2k}\left(1-\frac{1}{n}\right)^{n-2k}.\]
\end{lemma}
%
% \begin{proof}
% To keep the same fitness there must be exactly the same the number of wrong and good bits flipped. As there is exactly $d$ wrong bits and $n-d$ good ones, the formula above can be deduced.
% \end{proof}

\textbf{Stochastic domination.}

For two real random variables $X$ and $Y$ we say that $Y$ \emph{stochastically dominates} $X$ if for all $k\in\R$ we have $\Pr[Y \geq k] \geq \Pr[X\geq k]$. See~\cite{Doe18b} for a more detailed description of this concept.
In that case, we use the notation $X \preceq Y$. We use this notion to argue and make precise that batter parents generate better offspring. The following result is from \cite{Witt13}.

\begin{lemma}
\label{stochastic}
Let $x,y\in\E$ such that $f(x)\leq f(y)$. Let $X = \mathcal{M}(x)$ and $Y = \mathcal{M}(y)$. Then $f(X) \preceq f(Y)$.
\end{lemma}

We also use the following well-known fact.
\begin{lemma}
\label{lem:expectancyDomi}
Let $X$ and $Y$ be two random variables over $\N$ such that $X\preceq Y$. Then $E[X]\leq E[Y]$.
\end{lemma}

A \emph{coupling} for two random variables $X$ and $Y$ is a pair of random variables $(\Tilde{X},\Tilde{Y})$ defined over the same probability space such that $X$ and $\Tilde{X}$ as well as $Y$ and $\Tilde{Y}$ follow the same law. The following result is well-known.

\begin{theorem}
\label{thm:coupling}
  Let $X$ and $Y$ be two random variables. Then the two following statements are equivalents.
  \begin{itemize}
      \item[1)] $X\preceq Y$.
      \item[2)] There exists a coupling $(\Tilde{X},\Tilde{Y})$ such that $\Tilde{X}\leq\Tilde{Y}$.
  \end{itemize}
\end{theorem}

%
% \textbf{Chernoff bounds.}
% A Chernoff bound is a strong concentration inequality for a sum of of bounded independent variables.

% \begin{theorem}[Chernoff]
% \label{chernoff}
% Let $X_1,\cdots,X_n$ be independent random variables taking values in $[0,1]$. If $X = \sum\limits_{i=1}^n{X_i}$ then for all $\delta\geq0,$
% \[\Pr(X\leq E[X]-\delta) \leq \exp\left(-\frac{2\delta^2}{n}\right),\]
% \[\Pr(X\geq E[X]+\delta) \leq \exp\left(-\frac{2\delta^2}{n}\right).\]
% \end{theorem}

\textbf{Binomial distributions.}

We exploit the following estimate for binomial distributions.

\begin{theorem}
\label{exexp}
Let $m\geq 1,\frac{1}{m} < p \leq 1$ and $X\sim\Bin(m,p)$. Then

\[\Pr(X\geq E[X])>\frac{1}{4}.\]
\end{theorem}

See \cite{GreenbergM14} for a proof of this result.

\begin{theorem}
\label{log1plusBin}
There exists a constant $S_{min}$ such that if $X\sim\Bin(n,p)$ with $np\geq S_{min}$, then we have
\[E[\ln(1+X)]\geq \ln(1+np)-\frac{11}{12}\frac{(1-p)}{np}.\]
\end{theorem}

We omit the proof for reasons of space\footnote{The reviewer can find all proofs that were omitted for reasons of space in the optional appendix. We shall make a full version of this work available at the arXiv preprint server when the double-blind reviewing period is over.}.

\textbf{Martingales.}

Recall that a \emph{martingale} with respect to the filtration $\mathcal{F}$ is a stochastic process $M$ such that, for all $n\in\N$, we have $E[M_{n+1}\mid \mathcal{F}_n]=M_n$.

\begin{lemma}\label{lem:variation}
  Let $\lambda \ge e\mu$. Let $X_t \sim \min\{\mu, \Bin(\lambda, \frac{X_{t - 1}}{e\mu})\}$. Then for all $t \in \N$ and all $\Delta > 0$, the probability that for some $\tau \in [1..t]$ we have $X_\tau < X_0  - \Delta$ is at most $\frac{tX_0}{\Delta^2}$.
\end{lemma}
We omit the proof for reasons of space.

\textbf{Drift analysis.}
% To study the runtime of an evolutionary algorithm, it is common to use a drift analysis which consists on finding a good potential function which describes the state of the algorithm and then, by analyzing the expected difference in potential at each generation, deducing an upper bound of the expected runtime. In this subsection we present some very useful results from this perspective.

The application of the additive drift theorem~\cite{HeY01} can be difficult because it requires the analyzed process to be non-negative. For this reason we introduce the following lemma, which is more adapted to the processes studied in this paper.

\begin{lemma}\label{lem:dirty-trick} %Sorry for using such label. It's not dirty, but the word "dirty" is easy to find in our paper.
Let $\lambda \geq e\mu$. %, and for some small constant $\varepsilon$ we have $\lambda \le (1 + \varepsilon) e\mu$.
Let $X_t$ and $\Delta_t$ be some random processes such that for all $t \in \N$ we have $\Delta_t \geq \Delta_{\min}$ for some $\Delta_{\min} \in ]0, \lambda[$ and $X_{t + 1} \sim \min\{\mu, \Bin(\lambda, \frac{X_t + \Delta_t}{e\mu})\}$.
Let $T(X')$ be the first moment in time when $X_T \ge X'$ for some $X'$ that is at least $\max\{18\ln\frac{2\lambda}{\Delta_{\min}}, 48\}$, but not greater than $\frac{\mu}{2}$.
Then we have
\[
E[T(X')] \leq \max\left\{24, \frac{4X' - 2X_0}{\Delta_{\min}} \right\}.
\]
\end{lemma}
We omit the proof for reasons of space.

% New Section

% \section{A Lower Bound for $\frac{\lambda}{\mu}$.}

\section{Lower Bounds for $\lambda\leq(1-\varepsilon)\mu \lowercase{e}$ with $\varepsilon=\omega\left(\frac{1}{\sqrt{n}}\right)$.}

In this section, we show that when $\lambda\leq(1-\varepsilon)\mu e$ for some $\varepsilon=\omega\left(\frac{1}{\sqrt{n}}\right)$, then the runtime of the \ea on \onemax is super-polynomial. More precisely, our analysis reproves the exponential runtime shown in~\cite{Lehre10} for constant $\varepsilon$ and it enlarges the range for which a super-polynomial runtime is proven to $\varepsilon=\omega\left(\frac{1}{\sqrt{n}}\right)$.

\begin{theorem} %If $n$ is large enough, t
The following two statements hold about the expected runtime of the $(\mu,\lambda)$ EA on the $n$-dimensional \onemax function.
\label{thm:leq1minusEps}
\begin{enumerate}
    \item If there exists a constant $\varepsilon\in]0,1[$ such that $\lambda\leq(1-\varepsilon)\mu e$, then the expected runtime is exponential in $n$.
    \item If there exists $\varepsilon=\omega\left(\frac{1}{\sqrt{n}}\right)$ such that $\lambda\leq(1-\varepsilon)\mu e$, then the expected runtime is super-polynomial in $n$.
\end{enumerate}
\end{theorem}

To prove this result we use the lower bound version of the additive drift theorem with the potential function made precise in Definition~\ref{potentialG}. The potential of the population is the sum of the potentials of the individuals. The potential of an individual, roughly speaking, is exponential in its fitness. Due to this drastically increasing potential, we can estimate the potential of the next population via the potential of all offspring, including those who do not survive. By this, we circumvent the usually difficult analysis of the effects of selection.

Exploiting that fitness gains are rare when close to the optimum, we show that this potential has an expected increase (``drift'') of  at most $2\lambda$ per iteration. Again exploiting the strong growth of this potential function, we see that the potential difference of the initial population and any population containing the optimum is large, which gives the desired lower bound via the additive drift theorem.
%
%
%Finally, we argue that when the algorithm finds the optimum, the potential is at least some $H$ which is super-polynomial in $n$ and deduce by Theorem~\ref{thm:addDrift} that the expected runtime is super-polynomial.

\begin{definition}[Potential function]
\label{potentialG}
Let $\tau=\frac{4e}{\varepsilon}$,
\[\alpha=1-\frac{1}{\tau}\ln\left(1+\frac{1}{\tau}\right)\in\left]\tfrac 34,1\right[\]
and $f_0=\ceil{\alpha n}$. The potential function $g$ is defined over $\E$ by
\[
g(x)=\left\{
    \begin{array}{ll}
        \tau^{f(x)-f_0} & \mbox{if } f(x)\geq f_0, \\
        0 & \mbox{otherwise}.
    \end{array}
\right.
\]
For a population $P$, we define

\[g(P)=\sum\limits_{x\in P}g(x).\]
\end{definition}

The following key lemma estimates the drift of the potential in each generation.

\begin{lemma}
\label{lem:final6}
For all $t$, we have $E[g(P_{t+1})]\leq g(P_t)+2\lambda$.
\end{lemma}

To prove this result, we first compute the expected fitness of an offspring of a search point of fitness at least $f_0$. The proof is omitted for reasons of space.%\footnote{The reviewer can find all proofs that were omitted for reasons of space in the optional appendix. We shall make a full version of this work available at the arXiv preprint server when the double-blind reviewing period is over.}

\begin{theorem}
\label{potentialCap}
For any individual $x$ of fitness $f(x)\geq f_0$, we have

\[E[g(\mathcal{M}x)]\leq\frac{1}{e}(1+\varepsilon)g(x).\]
\end{theorem}

Since Theorem~\ref{potentialCap} applies when $f(x)\geq f_0$, we now show that the expected potential of an offspring of a search point of fitness at most $f_0-1$ is at most constant.

\begin{lemma}
\label{betterNeglected}
If $n$ is large enough, for all individuals $x$ such that $f(x)<f_0$, we have
\[E[g(\mathcal{M}x)]\leq 2.\]
\end{lemma}

Now we prove Lemma~\ref{lem:final6} using Theorem~\ref{potentialCap} and Lemma~\ref{betterNeglected}.

\begin{proof}[Proof of Lemma~\ref{lem:final6}]
If $U_t$ is a random individual chosen uniformly from $P_t$, then
\begin{align*}
    E[g(\mathcal{M}U_t)]&=\sum\limits_{x\in P_t}\frac{E[g(\mathcal{M}x)]}{\mu}\leq\sum\limits_{x\in P_t}\frac{\frac{1}{e}(1+\varepsilon)g(x)+2}{\mu}\\
    &=\frac{\frac{1}{e}(1+\varepsilon)g(P_t)}{\mu}+2.
\end{align*}

Let $\Tilde{P}_{t+1}$ be the set of the $\lambda$ offspring generated from $P_t$. Since $P_{t+1}\subset\Tilde{P}_{t+1}$, we have
\[E[g(P_{t+1})]\leq E[g(\Tilde{P}_{t+1})]\leq\frac{\lambda}{e\mu}(1+\varepsilon)g(P_t)+2\lambda.\]
We recall $\lambda\leq(1-\varepsilon)\mu e$, so that
\[\frac{\lambda}{e\mu}(1+\varepsilon)g(P_t)+2\lambda\leq (1-\varepsilon^2)g(P_t)+2\lambda\leq g(P_t)+2\lambda.\]
\end{proof}

To use the additive drift theorem we need a positive potential function that is equal to zero when the process is terminated. To define such a potential we note that if the algorithm has found the optimum $x^*$, then $g(P_t)\geq g(x^*)=\tau^{n-f_0}$. Thanks to this property, it is sufficient to show that the expected time for the potential to reach $\tau^{n-f_0}$ is exponential. This leads us to define $Z_t$, for all generation $t$, by
\[Z_t:=\max\{0,\tau^{n-f_0}-g(P_t)\}.\]

We also define $T':=\inf\{t\geq0 \mid Z_t=0\}$ and $\mathcal{S}$ by the common state space of all $Z_t$. Note that if $T$ is the expected runtime of the algorithm, we have $E[T']\leq E[T]$.

At this point we aim to know how $Z_t$ changes between two generations. As $\tau^{n-f_0}-g(P_t)$ is either negative or equal to $Z_{t+1}$, according to Lemma~\ref{lem:final6}, for all $t\geq0$ and for all $s\in\mathcal{S}\backslash\{0\}$, we have
\begin{align*}
    E[Z_t-Z_{t+1}\mid Z_t=s]&\leq E[g(P_{t+1})-g(P_t)\mid Z_t=s] \leq 2\lambda.
\end{align*}

Having verified the assumptions of the additive drift theorem, it remains to compute the initial potential $E[Z_0]$ via a simple Chernoff bound argument.

\begin{lemma}
\label{Ez0}
If $n$ is large enough, and if $\mu$ is sub-exponential in $n$, we have
\[E[Z_0]\geq \frac{1}{2}\tau^{n-f_0}.\]
\end{lemma}

Now we can prove the main result of this subsection.

\begin{proof}[Proof of Theorem~\ref{thm:leq1minusEps}]
If $\mu$ is super-polynomial, the expected runtime is also super-polynomial so we can assume that $\mu$ and $\lambda$ are at most polynomial. Recall that $\alpha=1-\frac{1}{\tau}\ln\left(1+\frac{1}{\tau}\right)$, that $\tau=\frac{4e}{\varepsilon}$ and that $f_0=\ceil{\alpha n}$.
By the additive drift theorem and Lemma~\ref{Ez0}, we have
\begin{align*}
    E[T]&\geq E[T']\geq\frac{1}{2}\frac{\tau^{n-f_0}}{2\lambda}\geq\frac{1}{4\tau\lambda}\exp\left(\frac{n}{\tau}\ln\left(1+\frac{1}{\tau}\right)\ln \tau\right)\\
    &\geq\frac{1}{4\tau\lambda}\exp\left(\frac{n}{\tau}\left(\frac{1}{\tau}-\frac{1}{2\tau^2}\right)\ln \tau\right)\\
    &\geq \frac{1}{4\tau\lambda}\exp\left(\frac{n}{2\tau^2}\ln \tau\right)=\frac{\varepsilon}{16e\lambda}\exp\left(\frac{n\varepsilon^2}{32e^2}\ln\frac{4e}{\varepsilon}\right).
\end{align*}
If $\varepsilon$ is a constant, then the expected runtime is exponential. If $\varepsilon=\omega\left(\frac{1}{\sqrt{n}}\right)$, two cases arise.

If $\varepsilon\geq n^{-\frac{1}{4}}$, we have
\[E[T]\geq \frac{1}{16e\lambda n^{\frac{1}{4}}}\exp\left(\frac{\sqrt{n}}{32e^2}\right),\]

whereas for $\varepsilon\leq n^{-\frac{1}{4}}$, we have
\[E[T]\geq\frac{1}{16e\lambda\sqrt{n}}\exp\left(\omega(1)\ln n\right).\]

In both cases the expected runtime is super-polynomial.
\end{proof}

\section{The Runtime when $\lambda\leq\mu \lowercase{e}$.}

\subsection{The Irrationality of $e$ and its Consequences on the Runtime.}

In this subsection, we assume that
\begin{equation}
\label{eq:lambdaleqemu1}
\left\lbrace
\begin{aligned}
    &\lambda\leq\mu e,\\
    &\mu\rightarrow\infty\textrm{ when }n\rightarrow\infty.
\end{aligned}
\right.
\end{equation}
The purpose of this subsection is to show that, under some conditions over $\mu$, the expected runtime of the algorithm is super-polynomial. Note that as $e$ is irrational, $\mu e-\lambda>0$. Thus, with $\varepsilon=\frac{\mu e-\lambda}{\mu e}$ we have $\lambda\leq(1-\varepsilon)\mu e$ so that if $\varepsilon=\omega\left(\frac{1}{\sqrt{n}}\right)$
we can apply Theorem~\ref{thm:leq1minusEps}. Our aim is to show the following theorem as a direct consequence of Theorem~\ref{thm:leq1minusEps}.

\begin{theorem}
\label{quater}
If the conditions $(\ref{eq:lambdaleqemu1})$ are met and if there exists a constant $c\in]0,\frac{1}{4}[$ such that $\mu\leq n^{\frac{1}{4}-c}$, then the expected runtime of the $(\mu,\lambda)$ EA on $n$-dimensional \onemax is super-polynomial in $n$.
\end{theorem}

For this purpose, we define the irrationality exponent.

\begin{definition}
\label{irrationalityMeasure}
Let $x\in\mathbb{R}$. We say that $x$ has an approximation of degree $d>0$ if the set of integers
\[\left\{p,q\in\N\textrm{ such that }\abs{x-\frac{p}{q}} \leq \frac{1}{q^d}\right\}\]
is infinite. The \emph{irrationality exponent} of $x$ is the least upper bound of the reals $d>0$ such that $x$ has an approximation of degree $d$.
\end{definition}

A cultural fact is that the exponent of irrationality of $e$ is known, and is given by the following theorem.

\begin{theorem}
\label{e2}
The irrationality exponent of $e$ is $2$.
\end{theorem}

\begin{proof}
See \cite[Theorem 1]{Dav78} for a proof of this result.
\end{proof}

A direct consequence of Theorem~\ref{e2} and Theorem~\ref{thm:leq1minusEps} is Theorem~\ref{quater}.
\begin{proof}[Proof of Theorem~\ref{quater}]
Let $c\in]0,\frac{1}{4}[$ such that $\mu\leq n^{\frac{1}{4}-c}$. We denote $\varepsilon=\frac{\mu e-\lambda}{\mu e}=\frac{1}{e}\abs{e-\frac{\lambda}{\mu}}$. Let $d=2+c$. Note that $d>2$ so, according to Theorem~\ref{e2}, the set defined in Definition~\ref{irrationalityMeasure} is finite. Because of condition (\ref{eq:lambdaleqemu1}), $\mu\rightarrow\infty$ when $n\rightarrow\infty$ so if $n$ is large enough, the couple $\lambda,\mu$ is not in the set. In other words
\[\abs{e-\frac{\lambda}{\mu}}\geq \frac{1}{\mu^d}\geq \frac{1}{n^{(\frac{1}{4}-c)(2+c)}}=\frac{n^{\frac{7c}{4}+c^2}}{\sqrt{n}}=\omega\left(\frac{1}{\sqrt{n}}\right).\]
Now, by Theorem~\ref{thm:leq1minusEps}, the expected runtime of the \ea on $n$-dimensional \onemax is super-polynomial in $n$.
\end{proof}

\subsection{The Super-polynomial Runtime for Low $\mu$.}

Now, we assume that the following conditions are met.

\begin{equation}
\label{eq:lambdaleqemu}
\left\lbrace
\begin{aligned}
    &\lambda\leq\mu e,\\
    &\textrm{for any } \varepsilon=\omega\left(\frac{1}{\sqrt{n}}\right),\textrm{ if $n$ is large enough then }\lambda\geq(1-\varepsilon)\mu e,\\
    &\textrm{there exists a constant } c\in]0,\frac{1}{2}[\textrm{ such that }\mu\leq n^{\frac{1}{2}-c},\\
    &\mu\rightarrow\infty\textrm{ when }n\rightarrow\infty.
\end{aligned}
\right.
\end{equation}

Our goal is to show the following theorem.

\begin{theorem}
\label{half}
If the conditions $(\ref{eq:lambdaleqemu})$ are met, then the expected runtime of the $(\mu,\lambda)$ EA on $n$-dimensional \onemax is super-polynomial in $n$.
\end{theorem}

For this purpose, we define the \emph{top level} $f_{\mathrm{top}}(t)$ at the generation $t$ as the best fitness among the population. Namely, $f_{\mathrm{top}}(t):=\max\{f(x), x\in P_t\}$. In this subsection, $X_t$ is the number of individuals of fitness $f_{\mathrm{top}}$ after the $t$-th generation and we write
\[h(P_t):=X_t\left(\ln\mu-\ln X_t+2\right).\]
We show that if $X_t$ is larger than some constant and if $f_{\mathrm{top}}$ is high, then $h$ has a constant drift towards $0$. By the additive drift theorem, we conclude that the algorithm has a constant probability to lose its top level in $O(\mu)$ generation. After this, we use the negative drift theorem on the top level itself to conclude.

To make sure that the top level decreases when $h$ reaches $0$, we assume that no good mutation occurs during $L$ generations in a row. To be more precise, for any $L \in \N$ we define $N_L$ as an event when during $L$ consecutive generations, the following two conditions are met.
\begin{enumerate}
    \item[1)] For all individuals $x$ of fitness $f\leq f_{\mathrm{top}}-1$ we have $f(\mathcal{M}x) < f_{\mathrm{top}}$,
    \item[2)] for all individuals $x$ of fitness $f_{\mathrm{top}}$, we have $f(\mathcal{M}x)\leq f_{\mathrm{top}}$ with $f(\mathcal{M}x) = f_{\mathrm{top}}$ if and only if $\mathcal{M}x = x$.
\end{enumerate}
We first compute the probability of $N_L$ then assume that $N_L$ is satisfied and deduce the actual drift.

\begin{lemma}
\label{lem:Pr(N_L)}
Assume that the conditions $(\ref{eq:lambdaleqemu})$ are met. Let $D=n^c$ and $L\in\N$. Then if $f_{\mathrm{top}}\geq n-D+1$ and if $n$ is large enough, we have
\[\Pr(N_L)\geq 1-\frac{eL}{\sqrt{n}}.\]
Moreover, the bounding $\Pr(N_1)\geq1-\frac{e}{\sqrt{n}}$ does not depend on the individuals chosen as the parents during the first generation.
\end{lemma}

\begin{proof}
Assume that, at generation $t$, the top level is at least $n-D+1$. Due to Lemma~\ref{stochastic}, we can assume that the whole population is in the two best fitness levels. Consequently, all that follows does not depend on the individuals picked as the parents during the first phase of the generation. For all individuals $x\in P_t$, $\M$ flips independently each bit of $x$, so there is a probability of at least $(1-\frac{1}{n})^D$ that none of the wrong bits of $x$ is mutated. Consequently, if we look at the first generation, there is a probability of at least $(1-\frac{1}{n})^{D\lambda}$ that the condition $N_1$ is satisfied. By iteration we show that
\[\Pr(N_L)\geq \left(1-\frac{1}{n}\right)^{D\lambda L}.\]
Hence, by Bernoulli's inequality we have
\begin{align*}
    \Pr(N_L)&\geq 1-\frac{D\lambda L}{n}\geq1-\frac{eL}{\sqrt{n}}.
\end{align*}
\end{proof}

Lemma~\ref{lem:Pr(N_L)} shows that $N_1$ is relatively likely. The following result refers to the law of $X_{t+1}$ when $N_1$ is satisfied.

\begin{lemma}
\label{lem:PrifN_l}
Assume that the conditions $(\ref{eq:lambdaleqemu})$ are met. Let $s\geq1$ and assume $X_t=s$. If condition $N_1$ is met, there exists a sequence $(p_n)_\N$ such that for all $n$, we have $\lambda p_n\leq s$ with $\lambda p_n\underset{n\to+\infty}{\longrightarrow}s$ uniformly and $X_{t+1}\sim\min\{\Bin(\lambda,p_n),\mu\}$.
\end{lemma}

\begin{proof}
In this proof we consider that the condition $N_1$ is met. We divide a generation into $\lambda$ independent phases, each phase consisting on the choice of an individual from $P_t$ and its mutation. Let $A$ be the event where an individual of fitness $f_{\mathrm{top}}$ is picked as a parent during the first phase and let $\Tilde{X}_{t+1}$ be the number of individuals in the top level before the selection.

Due to the condition $N_1$, the set of $\Tilde{X}_{t+1}$ individuals in the top level is exclusively made up of copies from the $X_t$ individuals in the previous generation. Consequently, if $p_n = (1-\frac{1}{n})^n\Pr(A\mid N_1)$ we have $\Tilde{X}_{t+1}\sim\Bin(\lambda,p_n)$. By Bayes' formula we have
\[\Pr(A\mid N_1)=\Pr(A)\frac{\Pr(N_1\mid A)}{P(N_1)}\geq\frac{s}{\mu}\Pr(N_1\mid A)\]
By Lemma~\ref{lem:Pr(N_L)}, we have $\Pr(N_1\mid A)\geq\left(1-\frac{e}{\sqrt{n}}\right)$, therefore
\[\Pr(A\mid N_1)\geq\frac{s}{\mu}\left(1-\frac{e}{\sqrt{n}}\right).\]
Moreover, since $x \le \ln(1 + x)$, by conditions~\eqref{eq:lambdaleqemu}, for all $n$ we have $\lambda p_n\leq s$. Finally, apply conditions (\ref{eq:lambdaleqemu}) with $\varepsilon=\frac{\ln n}{\sqrt{n}}$. As $s\leq\mu$ we have
\begin{align*}
    \lambda p_n &\geq \frac{\lambda s}{\mu}\left(1-\frac{e}{\sqrt{n}}\right)\left(1-\frac{1}{n}\right)^n\geq es\left(1-\frac{\ln n}{\sqrt{n}}\right)\left(1-\frac{e}{\sqrt{n}}\right)\left(1-\frac{1}{n}\right)^n\\
   % &= s\left(1-\frac{\ln n}{\sqrt{n}}\right)\left(1-\frac{e}{\sqrt{n}}\right)\left(1-\frac{1}{2n}+o\left(\frac{1}{n}\right)\right)\\
    &= s -\frac{s\ln n}{\sqrt{n}}+o\left(\frac{s\ln n}{\sqrt{n}}\right)\geq s-\frac{\ln n}{n^c}+o\left(\frac{\ln n}{n^c}\right).
\end{align*}
Therefore $\lambda p_n \underset{n\to+\infty}{\longrightarrow}s$ uniformly. Since $X_{t+1}\sim\min\{\Tilde{X}_{t+1},\mu\}$, $p_n$ is the desired sequence.
\end{proof}

The following theorem gives an interesting result about the drift.

\begin{theorem}
\label{boundedTopLevel}
If the conditions $(\ref{eq:lambdaleqemu})$ hold, there exists a constant $S\in\N$ and a constant $\beta\leq\frac{24e}{e-2}$ such that if the condition $N_1$ is satisfied and if $X_t\geq S$,
\[E[h(P_t)-h(P_{t+1})\mid P_t]\geq \frac{1}{\beta}.\]
\end{theorem}
For reasons of space we omit the proof.

\begin{corollary}
\label{loseifN}
If the conditions $(\ref{eq:lambdaleqemu})$ and $N_L$ are satisfied with $L:=4\beta\mu+1$ where $\beta$ is the constant from Theorem~\ref{boundedTopLevel} and if $n$ is large enough, then there is a probability of at least $\frac{e^{-S}}{3}$ that the top level decreases after $L$ generations.
\end{corollary}

\begin{proof}
In this proof, every probability and every expectancy are to be understood conditionally on $N_L$. Let $S$ be the constant from Theorem~\ref{boundedTopLevel} and $h':=h\mathds{1}_{X_t\geq S}$.
By Theorem~\ref{boundedTopLevel}, there is a constant drift $\frac{1}{\beta}>0$ such that, for all $s>0$,
\[E[h'(P_t)-h'(P_{t+1})\mid h'(P_t)=s]\geq \frac{1}{\beta}.\]
Therefore, the additive drift theorem shows that the expected number of generations before $h'=0$ (i.e. $X_t<S$) is at most $\beta E[h'(P_0)]\leq2\beta\mu$.

Consequently, by Markov's inequality there is a probability of at least $\frac{1}{2}$ that $X_t$ drops below $S$ after at most $4\beta\mu$ generations. Because of the condition $N_L$, we make sure that $X_t<S$ after that many generations with a probability of at least $\frac{1}{2}$. Then, by Lemma~\ref{lem:PrifN_l} there is a probability of
\[\left(1-p_n\right)^\lambda\geq \left(1-\frac{S}{\lambda}\right)^\lambda\]
that each of the individuals of the top level is lost in the next generation. Besides, because $S$ is a constant and because of condition $(\ref{eq:lambdaleqemu})$, $\lim\limits_{n\rightarrow\infty}\left(1-\frac{S}{\mu e}\right)^\lambda=e^{-S}$ so that, when $n$ is large enough,
\[\left(1-\frac{S}{\lambda}\right)^\lambda\geq \frac{2}{3}e^{-S}.\]
Overall, if $N_L$ holds there is a probability of at least $\frac{e^{-S}}{3}$ that the top level decreases after $L$ generations.
\end{proof}

From Lemma~\ref{lem:Pr(N_L)} and Corollary~\ref{loseifN} we deduce the following.

\begin{corollary}
\label{cor:PloseAfterL}
Recall $L=4\beta\mu+1$ where $\beta$ is the constant from Theorem~\ref{boundedTopLevel}. Suppose that the conditions $(\ref{eq:lambdaleqemu})$ are met and let $D=n^c$. Then, if $f_{\mathrm{top}}\geq n-D+1$ and if $n$ is large enough, the probability to lose the top level after $L$ generations is at least $\frac{e^{-S}}{4}.$
\end{corollary}

\begin{proof}
By Lemma~\ref{lem:Pr(N_L)} we have
\[\Pr(N_L)\geq1-\frac{eL}{\sqrt{n}}=1-\frac{4e\beta\mu+e}{\sqrt{n}}\geq 1-\frac{4e\beta}{n^c}-\frac{e}{\sqrt{n}}.\]
Consequently, if $n$ is large enough, $\Pr(N_L)\geq\frac{3}{4}$. Finally we conclude by Corollary~\ref{loseifN}.
\end{proof}

Now we note $Y_t:=\min\{n-f_{\mathrm{top}}(t),n^c\}$. It represents the distance between the best individual and the optimum. We define

$\phi(0)=0$ and, for all $t\in\N$,
\[
\phi(t+1)=\left\{
    \begin{array}{ll}
        \phi(t)+L & \mbox{if } Y_{\phi(t)+L}\geq Y_{\phi(t)}, \\
        \inf\{y\geq \phi(t)\mid Y_y<Y_{\phi(t)}\} & \mbox{otherwise},
    \end{array}
\right.
\]
and finally, $Z_t:=Y_{\Phi(t)}$. In other words, we divide the process into phases of variable lengths so that if the top level does not increase during the $L$ next generations, the phase length is $L$ generations. Otherwise, the phase is stopped as soon as the top level exceeds its value at the beginning of the phase. This way, during phase $t$ (between $Z_t$ and $Z_{t+1}$), there can only be one generation after which the top level goes from some $f_1\leq f_{\mathrm{top}}(\phi(t))$ to some $f_2>f_{\mathrm{top}}(\phi(t))$. Here, $\phi(t)$ is the generation when phase $t$ begins and $Z_t$ is the value of $Y$ at the beginning of phase $t$.

In order to use the negative drift theorem on $Z_t$, let $b(n)=n^c$ and $a(n)=0$. First, we have the following lemma.

\begin{lemma}
\label{kleap}
Suppose that the conditions $(\ref{eq:lambdaleqemu})$ are met. Then, for all $k\geq1$,
\[\Pr(Z_t-Z_{t+1}=k\mid Z_t<b(n))\leq L\lambda\frac{n^{-k(1-c)}}{k!}.\]
\end{lemma}

\begin{proof}
Let $x\in P_{\phi(t)}$. Suppose that $Z_t<b(n)$ so that $f_{\mathrm{top}}=f_{\mathrm{top}}(\phi(t))\geq n-n^c$. Let $r\geq0$ such that $f(x) = f_{\mathrm{top}}-r$ and let $d=n-f_{\mathrm{top}}\leq n^c$. According to Lemma~\ref{lem:mutlaw},
\begin{align*}
    \Pr(f(\mathcal{M}x)=f_{\mathrm{top}}+k)&=\Pr(\delta_x=r+k)
    \leq \dbinom{d+r}{r+k}\frac{1}{n^{r+k}}\\
    &=\dbinom{d}{k}\frac{1}{n^k}\frac{(d+r)(d+r-1)\cdots(d+1)}{n^r(k+r)\cdots(k+1)}\\
    &\leq \dbinom{d}{k}\frac{1}{n^k}\leq \frac{1}{k!}\left(\frac{d}{n}\right)^k\leq \frac{n^{-k(1-c)}}{k!}.
    \end{align*}
Consequently, due to Bernoulli's inequality, if $f_{\mathrm{top}}\geq n-n^c$ the probability to leap from any top level $f\leq f_{\mathrm{top}}$ to top level $f_{\mathrm{top}}+k$ in one generation is at most
\[\lambda\frac{n^{-k(1-c)}}{k!}.\]
Now, due to the definition of $Z$, if $Z_t-Z_{t+1}=k$ then there is one generation in the $L$ generations where the jump takes place, so that
\[\Pr(Z_t-Z_{t+1}=k\mid Z_t<b(n))\leq L\lambda\frac{n^{-k(1-c)}}{k!}.\]
\end{proof}

Thanks to Lemma~\ref{kleap} we deduce that the conditions of the negative drift theorem hold.

\begin{lemma}
\label{exponentialMoment}
Suppose that the conditions $(\ref{eq:lambdaleqemu})$ are met. Let $S$ and $\beta$ be the constants from Theorem~\ref{boundedTopLevel}, let $L=4\beta\mu+1$  and $\Lambda:=c\ln n-S-\ln(40\beta)-1$. Then, if $n$ is large enough
\[E[e^{\Lambda(Z_t-Z_{t+1})}\mid a(n)<Z_t<b(n)]\leq 1-\frac{e^{-S}}{12}.\]
\end{lemma}

\begin{proof}
To ease the notation, we denote by $\Pr$ and $E$ the conditional probability and expectation conditional on $a(n)<Z_t<b(n)$ respectively. By Lemma~\ref{kleap}, we have
\begin{align*}
    \Sigma_+:=\sum\limits_{k\geq1}e^{k\Lambda}\Pr(Z_t-Z_{t+1}=k)&\leq\sum\limits_{k\geq1}\lambda L\frac{e^{k\Lambda} n^{-k(1-c)}}{k!}\\
    &\leq \lambda L\left(\exp\left(\frac{e^\Lambda}{n^{1-c}}\right)-1\right).
\end{align*}
Note that $\frac{e^\Lambda}{n^{1-c}}=O\left(\frac{1}{n^{1-2c}}\right)$ so because $c<\frac{1}{2}$, if $n$ is large enough, we have
\[\Sigma_+\leq 5e\beta\mu^2 \frac{e^\Lambda}{n^{1-c}}\leq \frac{5e\beta e^\Lambda}{n^c}=\frac{e^{-S}}{8}.\]
Besides, by Corollary~\ref{cor:PloseAfterL},
\[\Sigma_0:=\Pr(Z_t=Z_{t+1})\leq 1-\frac{e^{-S}}{4}.\]
Finally,
\begin{align*}
    \Sigma_-&:=\sum\limits_{k<0}e^{k\Lambda}\Pr(Z_t-Z_{t-1}=k)\\
    &\leq \sum\limits_{k<0}e^{-\Lambda}\Pr(Z_t-Z_{t-1}=k)\leq e^{-\Lambda}.
\end{align*}
If $n$ is large enough, $e^{-\Lambda}\leq\frac{e^{-S}}{24}$ thus
\[E[e^{\Lambda(Z_t-Z_{t+1})}]=\Sigma_-+\Sigma_0+\Sigma_+\leq 1-\frac{e^{-S}}{12}.\]
\end{proof}

Due to this result, we can apply the negative drift theorem and conclude.

\begin{proof}[Proof of Theorem~\ref{half}]
Let $T(n):=\inf\{t\geq0,Z_t=0\}$. Note that the expected runtime of the algorithm is at least $E[T(n)]$. Now, let $B(n):=\exp\left(\frac{\Lambda n^c}{2}\right)$. According to Lemma~\ref{exponentialMoment} and by the negative drift theorem, we have
\[\Pr(T(n)\leq B(n)\mid Z_0\geq n^c)\leq 12e^S\exp\left(-\frac{1}{2}\Lambda n^c\right).\]
As $\Lambda=\Theta(\ln n)$, the expected runtime is super-polynomial.
\end{proof}

\section{Polynomial Runtime on the Threshold for the Large Population Sizes}
\label{sec:poly-large-mu}

In this section we reveal a tighter threshold for the parent and offspring population sizes of the \ea that guarantees a polynomial runtime for the optimization of \onemax. We consider $\lambda \geq e\mu$ and $\mu = \omega(n^{\frac 23}\log^4(n))$.
Such relatively large values of $\mu$ give us a high concentration of several random variables such as the number of the individuals on the top level. This concentration turns out to be enough for even small drifts to play a significant role.

% \da{Benjamin, could you please see the following definition?}
In this subsection we define level $i$ as the set of all bit strings of length $n$ with exactly $i$ one-bits. We denote by $X_t(f)$ the number of individuals in $P_t$ of fitness exactly $f$ and by $Y_t(f)$ the number of individuals in $P_t$ that have a fitness strictly greater than $f$.

We say that the \emph{current level} is $f$ at generation $t$ if there exists $t_0 < t$ such that $X_{t_0}(f) + Y_{t_0}(f) \ge \frac{\mu}{2}$ and for all $\tau \in [t_0..t]$ we have $X_{\tau}(f) + Y_{\tau}(f) \ge \frac{\mu}{4}$ and $Y_{\tau}(f) < \frac{\mu}{2}$.
In other words, it is the lowest fitness level such that once there were at least $\frac{\mu}{2}$ individuals on this level or above, and since then this number of individuals has not fallen below $\frac{\mu}{4}$.

The current level $f$ can change in one of the two events. (i) There are less than $\frac{\mu}{4}$ individuals with fitness at least $f$ in the population (then we say that \emph{the algorithm loses a level}) or (ii) there are at least $\frac{\mu}{2}$ individuals with fitness more than $f$ in the population (then we say that \emph{the algorithm gains a level}).

For brevity we define $X_t \coloneqq X_t(f)$ and $Y_t \coloneqq Y_t(f)$, if the current level is $f$. The main result of this section is the following theorem.

\begin{theorem}\label{thm:poly-large-mu}
  If $\lambda \geq e\mu$ and $\mu \ge n^{\frac 23}\ln^4(n)$ and $\frac{\lambda}{\mu}$ is at most polynomial in $n$
  then the expected number of generations of the \ea on the \onemax function is at most $O(n\log(n))$.
  % \da{ check, if we need polynomial $\lambda$ somewhere.}
  % \qy{I would rather say that the expected number of generations of the \ea on the \onemax function is at most $O(n\ln n)$}
  % \qy{The expected rutime is not the expected number of generations.} \da{It depends on how we define in (and we define it as a number of generations, as far as I understand). And we specify here that runtime is in generations} \da{I have checked, and it seems like we do not define it at all. I would stick to the "number of generations" definition throughout the paper and specify it in.}\qy{I have done it in the last update. Now the runtime is the number of fitness evaluation, and is different from the number of generations. Both Ben and I agree that this should be the way to do it. If you want to stick to the number of generations I am fine with this, simply specify that you do so and make sure that you do not use the "runtime" terminology.}
\end{theorem}

To prove Theorem~\ref{thm:poly-large-mu}, we split the runtime into two phases. On the first phase, the current level is at most $\frac{n}{3}$.
On the second phase the the current level $f$ is greater than $\frac {n}{3}$. This allows us to show that the algorithm gains a level in expected number of  $O(\frac{n}{n - f})$ generations. At the same time there is only a small probability that the algorithm looses a level before it gains one. We first analyze these two phases separately and then we come up with the proof of Theorem~\ref{thm:poly-large-mu}.

\begin{lemma}\label{lem:large-mu-phase-1}
  The expected number of generations that the \ea spends on first phase is $O(n)$.
\end{lemma}

We omit the proof for reasons of space here.

To analyze the expected runtime of the second phase, we regard in details the behaviour of the algorithm on the current level $f$. The algorithm is not likely to lose the current level in a polynomial time. At the same time, it creates enough offspring with fitness at least $f + 1$ to fill the upper level in expected number of $O(\frac{n}{n - f})$ generations. The first observation leads us to the following lemma.

\begin{lemma}\label{lem:large-mu-stay-forever}
Let the current level at generation $0$ be $f > \frac n3$ and $\mu = n^{\frac 23} h(n)$, where $h(n) \ge \ln^4(n)$.
If $X_0 \ge \frac \mu2$ and $\lambda \ge e\mu$ and $\frac{\lambda}{\mu}$ is at most polynomial in $n$ and $n$ is large enough then for any $t \in \N $ we have $\Pr[X_t \geq \frac{\mu}{4}] \geq (1 - \frac{2}{n^3})^t$. %TODO in camera-ready: change to "for any \tau < t X_\tau \ge \mu/4"
\end{lemma}

Here we omit the strict proof for reasons of space, but present only the following sketch. $X_t$ performs an unbiased random walk with steps of size $O(\sqrt{X_t})$. For this reason the expected number of generations before $X_t \le \frac{\mu}{4}$ is linear in $\mu$. However, while $X_t \ge \frac{\mu}{4}$ we have a positive drift of order $\Theta(\frac{\mu}{n})$ for $Y_t$, that makes us reach $Y_t = \omega(\sqrt{\mu})$ in $o(\mu)$ iterations with high probability.

Once $Y_t = \omega(\sqrt{\mu})$, it is not likely to decrease by a factor more than $2$ in $\sqrt{\mu}$ iterations, since it preforms a random walk of the same manner as $X_t$ did. At the same time such great $Y_t$ creates an influx of individuals of fitness $f$ that is of greater order than the steps made by $X_t$.
This is enough for $X_t$ to become at least $\frac{\mu}{2}$ again before $Y_t$ becomes too small. This regular refilling of level $f$ does not let $X_t$ fall below $\frac{\mu}{4}$ for long enough.

Next observation is that before the algorithm loses a level, $Y_t$ has decent positive drift. This gives us the following lemma.

\begin{lemma}\label{lem:large-mu-phase-2}
If $\lambda \ge e\mu$ and $\frac{\lambda}{\mu}$ is at most polynomial in $n$ and $\mu = n^{\frac 23} h(n)$ where $h(n)\geq\ln^4(n)$, then the expected runtime before the \ea either loses or gains a level is at most $\frac{8n}{n - f}$ generations. The probability that this results in a level loss is at most $\frac{10}{n}$.
\end{lemma}

\begin{proof}
  Note that $Y_{t + 1} \succeq \min\left\{\mu, \Bin(\lambda, \frac{Y_t}{e\mu} + \frac{(n - f)X_t}{en\mu})\right\}$. Before the algorithm loses a level we have $\frac{X_t}{n} \ge \frac{\mu}{4n}$.
  By Lemma~\ref{lem:dirty-trick}, denoting $\Delta_t \coloneqq \frac{(n - f)X_t}{n} \ge \frac{(n - f)\mu}{4n}$ and $X' = \frac{\mu}{2}$, we have that the expected runtime before $Y_t \ge X'$ is at most $\frac{4X'}{\Delta_{\min}} = \frac{8n}{n - f}$.
  %\qy{Maybe just say either $2\mu*\frac{4n}{(n-f)\mu}$ or $\frac{8n}{n-f}$ to avoid this / thing}

  The probability that the algorithm gains a level is at least the probability that before generation $\tau \coloneqq n^2$ the algorithm has not lost a level multiplied %\qy{why? Are those two independents??}\da{Since we say "before this generation", we are conditional on that we do not lose a level.}
  by the probability that it has gained a level before this generation. By Lemma~\ref{lem:large-mu-stay-forever} this is at least

  \[
  \left(1 - \frac{2}{n^3}\right)^\tau = \left(1 - \frac{2}{n^3}\right)^{n^2} \ge \exp\left(-\frac{2}{n}\right) \ge 1 - \frac{2}{n}.
  \]

  By Markov's inequality the probability that the algorithm does not gain a level in $\tau$ generations is at most $\frac{8n}{(n - f)n^2} \le \frac{8}{n}$.

  Therefore, the probability that the algorithm loses a level before it gains one is at most $1 - (1 - \frac{2}{n})(1 - \frac{8}{n}) \le \frac{10}{n}.$
\end{proof}

Now we are ready to prove Theorem~\ref{thm:poly-large-mu}.

\begin{proof}[Proof (Theorem~\ref{thm:poly-large-mu})]
If the algorithm does not lose a level, then the expected number $T'$ of generations before it finds the optimum is at most the expected number of generations spent in the first phase plus the expected number of generations spent in each level of the second phase. By Lemmas~\ref{lem:large-mu-phase-1} and~\ref{lem:large-mu-phase-2} we have
\begin{align*}
  E[T'] \le O(n) + \sum_{f = \frac n3}^{n - 1} \frac{8n}{n - f} = O(n\log(n)).
\end{align*}

By Lemma~\ref{lem:large-mu-phase-2} the probability not to lose a level before reaching the optimum is not greater than $(1 - \frac{8}{n})^{\frac n3} \ge e^{-3}$, if $n$ is large enough. Since we pessimistically assume that in the event of a level loss, the algorithm goes back to level zero, losing a level is equivalent to a restart of the algorithm. However, the expected number of such restarts is not greater than $e^3$, so the total expected number of generations of the \ea on the \onemax function is $O(n \log(n)).$
\end{proof}

\section{Conclusion}

In this work, we have analyzed how the \ea optimizes the \onemax function when the population sizes are chosen close to the efficiency threshold $\lambda \approx e \mu$. This regime is interesting in that there is no clear negative drift, which strongly prevents approaching the global optimum, and in that there is no clear positive drift, which destroys the ability of comma selection to leave local optima (by creating with high probability a copy of the parent population).

Due to the technical challenges in this regime, this first analysis is not fully conclusive, and in fact, we observe that now also the absolute population size plays a role (more than just the need to be at least polynomial). Our results show in particular that close to the threshold, a polynomial runtime is still possible if the population size is not too small (but $n^{2/3 +\eps}$ is enough).

This raises the question (and hope) whether in this regime the \ea can overcome premature convergence when optimizing multi-modal optimization problems. Our upper bound proof suggests that in this regime the population is not quickly concentrated on the best-so-far fitness level, but is spread over more than one level. This could ease leaving such a local optimum. Since the analysis of the \ea on multi-modal problems is again a topic little understood, we cannot answer this question easily, but suggest this as an interesting problem for future research.

\section*{Acknowledgements}
Denis Antipov was supported by the Government of Russian Federation (Grant 08-08).
This work was supported by the Paris Ile-de-France Region.

\appendix
\newpage

\section{Appendix}

In this auxiliary material we present some famous results that we refer to as well as the proofs of some of our statements that had to be omitted for reasons of space.

\begin{theorem}[Chernoff Bounds]
\label{chernoff2}
Let $X_1,\cdots,X_n$ be independent random variables taking values in $[0,1]$. If $X = \sum\limits_{i=1}^n{X_i}$ then for all $\delta\in[0,1],$
\[\Pr(X\leq (1-\delta)E[X]) \leq \exp\left(-\frac{\delta^2}{2}E[X]\right),\]
\[\Pr(X\geq (1+\delta)E[X]) \leq \exp\left(-\frac{\delta^2}{3}E[X]\right).\]
\end{theorem}
See \cite{Hoe63} for the proof of these bounds.

\begin{theorem}[Markov's Inequality]
  For any non-negative random variable $X$ and any positive $a \in \R$ we have \[\Pr[X \ge a] \le \frac{E[X]}{a}.\]
\end{theorem}

\begin{theorem}[Additive Drift Theorem]
\label{thm:addDrift}
Let $(Z_t)_\N$ be a sequence of a non-negative random variables over a finite state space $\mathcal{S}$ which contains $0$. Let $T:=\inf\{t\geq0 \mid Z_t=0\}$.

\begin{enumerate}
    \item Suppose there exists a constant $\delta>0$ such that, for all $t\in\N$ and for all $s\in\mathcal{S}\backslash\{0\}$,
\[E[Z_{t}-Z_{t+1}\mid Z_t=s]\geq\delta.\]
Then
\[E[T]\leq\frac{E[X_0]}{\delta}.\]
    \item Suppose there exists a constant $\delta>0$ such that, for all $t\in\N$ and for all $s\in\mathcal{S}\backslash\{0\}$,
\[E[Z_{t}-Z_{t+1}\mid Z_t=s]\leq\delta.\]
Then
\[E[T]\geq\frac{E[X_0]}{\delta}.\]
\end{enumerate}
\end{theorem}

This statement of the additive drift theorem is taken from~\cite[Theorem 1]{Len18} that is an adaptation of the original version from~\cite{HeY01}

\begin{theorem}[Negative Drift Theorem]
\label{negativeDrift}
Let $(Y_t)_{t\in\N}$ be some real random variables and $n$ be some parameter. Let $a(n),b(n)$ be two reals such that $a(n)<b(n)$. Let
\[T(n):=\inf\{t\geq0\mid Y_t\leq a(n)\}.\]
If there exists $\Lambda(n)>0$ and $p(n)\geq1$ such that, for all $t\geq0$,
\[E[e^{\Lambda(n)(Y_t-Y_{t+1})}\mid a(n)<Y_t<b(n)]\leq 1-\frac{1}{p(n)},\]
then, for all $L(n)>0$, we have
\[\Pr(T(n)\leq L(n)\mid Y_0\geq b(n))\leq L(n)D(n)p(n)e^{-\Lambda(n)(b(n)-a(n))}\]
where
\[D(n)=\max\left\{1,E\left[e^{-\Lambda(n)(Y_{t+1}-b(n))}\mid Y_t\geq b(n)\right]\right\}.\]
\end{theorem}

This statement of the negative drift theorem is taken from~\cite{OW12} where it was adapted from the original in~\cite{Hajek82}.

\begin{theorem}[Doob's Decomposition]
  \label{thm:doob}
  For any integrable process $(X_n)_\N$, there exists a martingale $(M_n)_\N$ and a predictable integrable process $(A_n)_\N$ such that $A_0=M_0=0$ and, for all $n\in\N$, $X_n=X_0+M_n+A_n$. This decomposition is almost surely unique.
\end{theorem}

See \cite{Doob53} for a statement and a proof of this result.

\begin{lemma}
\label{logconcave}
Let $x>-1$. Then, $\ln(1+x)\leq x$.
\end{lemma}

% \da{we do not use the following theorem}\qy{According to my ctrl+f tool we use it twice.}
% \begin{theorem}
% \label{Bvariant}
% Let $m\in\N$ and $p_1,\cdots,p_m\in[0,1]^m$. Then,
%
% \[\prod\limits_{k=1}^m{(1-p_k)}\geq 1-\sum\limits_{k=1}^mp_k.\]
% \end{theorem}
%
% \begin{proof}
% This Lemma is trivial when $m=0$. Let's assume it's true for some $m\in\N$. Then, as $(1-p_{m+1})\geq0$, using the recurrence hypothesis we have
% \begin{align*}
%     \prod\limits_{k=1}^{m+1}(1-p_k)&=(1-p_{m+1})\prod\limits_{k=1}^m(1-p_k)\geq (1-p_{m+1})(1-\sum\limits_{k=1}^mp_k)\\
%     &=1-\sum\limits_{k=1}^{m+1}{p_k}+p_{m+1}\sum\limits_{k=1}^m{p_k}\geq1-\sum\limits_{k=1}^{m+1}{p_k}.
% \end{align*}
%
% \end{proof}
% This theorem is actually a variant of the Beroulli's inequality.

\begin{lemma}[Bernoulli's inequality]
\label{Bernoulli}
Let $x\geq-1$ and $m\in\N$. Then
\[(1+x)^m\geq1+mx.\]
\end{lemma}

\begin{proof}[Proof of Theorem~\ref{log1plusBin}]
We have
\[E[\ln(1+X)]=\ln(1+np)+E\left[\ln\left(1+\frac{X-np}{1+np}\right)\right].\]
By Chernoff bounds, for all $\delta\in[0,1]$ we have
\[\Pr(X\leq(1-\delta)np)\leq \exp\left(-\frac{\delta^2}{2}np\right).\]
Let $\delta\in]0,1[$ and $D:=\left(X\geq(1-\delta)np\right)$. Let $Y:=\frac{X-np}{1+np}$.

We have
\[E[\ln(1+Y)]=E[\mathds{1}_D\ln(1+Y)]+E[(1-\mathds{1}_D)\ln(1+Y)].\]

When $D$ is satisfied, we can use a Taylor expansion of the logarithm.
\[R:x\mapsto\ln(1+x)-x+\frac{x^2}{2}\]
is monotonically increasing on $]-1,+\infty[$ and, if $D$ is satisfied, we have $Y_n\geq -\delta$, so

\[E[\mathds{1}_D\ln(1+Y)]\geq E\left[\mathds{1}_D\left(Y-\frac{1}{2}Y^2+R(-\delta)\right)\right].\]
Due to the definition of $D$, the multiplication by $\mathds{1}_D$ is only cutting negative values, thus
\begin{align*}
    E[\mathds{1}_D\ln(1+Y)]&\geq E[Y]-\frac{1}{2}E[Y^2]+R(-\delta)\\
    &=-\frac{np(1-p)}{2(1+np)^2}+R(-\delta).
\end{align*}
By the Taylor-Laplace theorem with integral remainder we have
$R(-\delta)\geq -\frac{\delta^3}{3(1-\delta)^3}$.
With
$\delta\coloneq\frac{(1-p)^{\frac{1}{3}}}{(1-p)^{\frac{1}{3}}+(np)^{\frac{1}{3}}},$
we have
\[E[\mathds{1}_D\ln(1+Y)]\geq -\frac{np(1-p)}{2(1+np)^2}-\frac{1-p}{3np}.\]

Since $x\mapsto \ln(1+x)$ is monotonically increasing, we have
\[E[(1-\mathds{1}_D)\ln(1+Y)]\geq \ln\left(1-\frac{np}{1+np}\right)\exp\left(-\frac{\delta^2}{2}np\right).\]
This logarithm grows in $np$ against a super-polynomial decrease in $np$. So, there exists a constant $S_{min}$ such that if $np\geq S_{min}$, we have
\[E[(1-\mathds{1}_D)\ln(1+Y)]\geq -\frac{1-p}{12np}.\]
Consequently
\[E[1+X]\geq \ln(1+np)-\frac{11}{12}\frac{(1-p)}{np}.\]
\end{proof}

\begin{proof}[Proof of Lemma~\ref{lem:variation}]
% For convenience only we assume that $t=0$ and $X_0=k$ for some $k$.
We define the process $M_\tau$ as follows. $M_0 = X_0$ and, for $\tau \geq 0$ and $M_{\tau+1}\sim\Bin\left(\lambda,\frac{M_\tau}{\lambda}\right)$ if $M_\tau > X_0 - \Delta$ and $M_{\tau+1} = M_\tau$ otherwise.
By Doob's decomposition theorem (Theorem~\ref{thm:doob}), there exists a martingale $(N_\tau)_\N$ and a predictable process $(A_\tau)_\N$ such that, for all $\tau\in\N$, $(M_\tau)^2=(M_0)^2+N_\tau+A_\tau$. Note that $M$ is a martingale. Consequently, for all $\tau\geq0$, we have
\begin{align*}
    E[(M_{\tau+1}-M_\tau)^2\mid \mathcal{F}_\tau] &= E[(M_{\tau+1})^2-2M_{\tau+1}M_\tau+(M_\tau)^2\mid \mathcal{F}_\tau]\\
    &= E[(M_{\tau+1})^2-(M_\tau)^2\mid \mathcal{F}_\tau]=A_{\tau+1}-A_\tau.
\end{align*}
We sum these equalities to obtain
\[\sum\limits_{k=0}^{\tau-1}E[(M_{k+1}-M_k)^2\mid \mathcal{F}_k]= A_\tau=(M_\tau)^2-(M_0)^2-N_\tau.\]
Now as $N$ is a martingale with $N_0=0$ we have $E[N_\tau]=E[N_0]=0$, therefore
\[E[(M_\tau)^2]= E[(M_0)^2]+\sum\limits_{k=0}^{\tau-1}E[(M_{k+1}-M_k)^2].\]
Finally, $M$ is a martingale so $E[M_\tau]^2= E[M_0]^2 = E[(M_0)^2]$, since $X_0$ is determined. Thus we have
\begin{align*}
    \Var(M_\tau)&\leq \sum\limits_{k=0}^{\tau-1}E[(M_{k+1}-M_k)^2]=\sum\limits_{k=0}^{\tau-1}E[E[(M_{k+1}-M_k)^2\mid \mathcal{F}_k]]\\
    &=\sum\limits_{k=0}^{\tau-1}E[\Var[M_{k + 1} \mid \mathcal{F}_k]] \leq \sum\limits_{k=0}^{\tau-1}E\left[M_k\left(1-\frac{M_k}{\lambda}\right)\right] \\
    &\leq \sum\limits_{k=0}^{\tau-1}E[M_k]= \tau X_0.
\end{align*}
By Chebyshev's inequality,
\[\Pr[M_\tau\leq X_0-\Delta\mid X_0=k] \leq \frac{\tau k}{\Delta^2}.\]
% Now recall that if $f\leq n-1$, an individual $x$ chosen at random at the $(t+1)-th$ generation has a probability of $\frac{Z_t}{\mu}$ to have a fitness of at least $f$, in which case its mutant has a probability of at least $\left(1+\frac{1}{n}\right)^n+\frac{1}{n}\geq \frac{1}{e}$ to have a fitness of at least $f$.
% As $\lambda\geq e\mu$, if $X_t=s$ then the number of individuals of fitness at least $f$ before the selection stochastically dominates $\Bin\left(\lambda,\frac{s}{\lambda}\right)$.
Therefore as long as $M_\tau\in\{X_0 - \Delta + 1,\cdots,\mu\}$ we have $M_\tau\preceq X_\tau$. Let $\tau_1=\inf\{t\in\N\mid M_t>\mu\}$.
By Theorem~\ref{thm:coupling}, there exists a coupling $(\tilde X,\tilde M)$ such that for all $\tau < \tau_1$, we have $\tilde M_\tau \leq \tilde X_\tau$. Therefore if $\Tilde{M}_\tau$ exceeds $\mu$ we can wait until $\tilde X_\tau \leq X_0$ and restart the argument with $t=\tau_1$ and $\tau'=\tau-\tau_1\leq\tau$. Finally we have
\[\Pr[\exists t \in [0..\tau]: X_t\leq X_0-\Delta\mid X_0=k]\leq \frac{\tau k}{\Delta^2}.\]

\end{proof}

\begin{proof}[Proof of Lemma~\ref{lem:dirty-trick}]
Although the drift of $X_t$ towards $X'$ is at least $\Delta_t$, we cannot apply the additive drift theorem from the box, since this drift partially comes from the fact that $X_t$ is surely larger than $X'$.

To overcome this problem we define the potential function $\Phi(X)$ for all $X \in \N$ as follows.
\begin{align*}
  \Phi(X) =
  \begin{cases}
    0, \text{ if } X \ge X', \\
    2X' - X, \text{ else.}
  \end{cases}
\end{align*}
% Recall that $X_t$ takes values only in $[0,\cdots,\lambda]$ and its expected value is $E[X_t] \le X_{t - 1} + \Delta_t$. Hence, we estimate the expected difference in the potential function after one step of the process as follows.

% Recall that $X_t$ takes values only in $[0..\lambda]$ and its expected value is $E[X_t] \le X_{t - 1} + \Delta_t$.
To ease the notation we introduce another random process $\tilde X_t \sim \Bin(\lambda, \frac{X_t + \Delta_t}{e\mu})$.
Note that $E[\tilde X_t] =  X_{t - 1} + \Delta_{t - 1}$. We also define $p_t(i) \coloneqq \Pr[X_t = i \mid X_{t - 1}]$ and $q_t(i) \coloneqq \Pr[\tilde X_t = i \mid X_{t - 1}]$. Note that if $i < \mu$, then $p_t(i) = q_t(i)$ and $p_t(\mu) = \sum_{i = \mu}^\lambda q_t(i)$.
% \da{please check if this really eases the notation. I found keeping all these $\Pr[X_t = i \mid X_{t-1}]$ too heavy.}\qy{It is visually better but it is still hard to understand what is actually going on.}
Hence, we estimate the expected difference in the potential function after one step of the process as follows.

\begin{align*}%\label{eq:dirty-trick-delta-phi}
E[\Phi(X_t) - \Phi(X_{t + 1}) &\mid X_t] = \sum_{i = 0}^{X' - 1} p_{t + 1}(i) (i - X_{t}) \\
                               &+ \sum_{i = X'}^{\mu} p_{t + 1}(i) (2X' - X_{t})  \\
                            &\geq \sum_{i = 0}^{2X'} q_{t + 1}(i) i              + \sum_{i = 2X' + 1}^{\lambda} q_{t + 1}(i) 2X' - X_t \\
                            &\geq \sum_{i = 0}^{\lambda} q_{t + 1}(i) i             + \sum_{i = 2X' + 1}^{\lambda} q_{t + 1}(i) (2X' - i) - X_t \\
                            &\geq E[\tilde X_{t + 1}] - X_t - \sum_{i = 2X' + 1}^{\lambda} q_{t + 1}(i) i\\
                            &\geq \Delta_t - \lambda \Pr[\tilde X_{t + 1} > 2X'].
\end{align*}

If $\Delta_t \leq \frac{X'}{2}$, then by Chernoff bounds we have

\begin{align*}
  \Pr[\tilde X_{t + 1} \geq 2X'] &= \Pr\left[\tilde X_{t + 1} \geq \left(1 + \frac{2X'}{X_t + \Delta_t} - 1\right)(X_t + \Delta_t)\right] \\
                              &\leq \exp\left( -\left(\frac{2X'}{X_t + \Delta_t} - 1\right)^2 \frac{X_t + \Delta_t}{3}\right) \\
                              &\leq \exp\left( -\frac{(X' - \Delta)^2}{3(X' + \Delta)}\right) \leq \exp\left( -\frac{X'}{18}\right).
\end{align*}
Since by the lemma conditions we have $X' \geq 18\ln\frac{2\lambda}{\Delta_{\min}}$, we have $\Pr[\tilde X_{t + 1} > 2X'] \leq \frac{\Delta_{\min}}{2\lambda}$. Hence we obtain

\[
E[\Phi(X_t) - \Phi(X_{t + 1}) \mid X_t] \geq \Delta_t - \lambda \frac{\Delta_{\min}}{2\lambda} \geq \frac{1}{2}\Delta_{\min}.
\]

Otherwise, if $\Delta_t \geq \frac{X'}{2} > 1$, we have $\frac{(X_t + \Delta_t)}{e\mu} > \frac 1\lambda$ and thus by Theorem~\ref{exexp} we have $\Pr[X_{t + 1} \geq X_t + \Delta_t] \ge \tfrac{1}{4}$. At the same time by Chernoff bounds we have $\Pr[X_{t + 1} < X_t] \le \exp(-\frac{X'}{6})$.
 Hence, the drift of the potential function is at least $\frac{\min\{X', \Delta_{\min}\}}{4} - X'\exp(-\frac{X'}{6}) \ge \frac{X'}{8} - 3 \ge \frac{X'}{12}$.

Finally, applying the additive drift theorem (Theorem~\ref{thm:addDrift}) we have
\begin{align*}
E[T(X')] &\leq \frac{\Phi(X_0)}{\min\{\frac{1}{2}\Delta_{\min}, \frac{1}{12}X'\}} \leq \max\left\{\frac{4X' - 2X_0}{\Delta_{\min}}, 24\right\}.
\end{align*}
\end{proof}

\begin{proof}[Proof of Theorem~\ref{potentialCap}]
Let $x$ be an individual such that $f(x)\geq f_0$ and $d=n-f(x)$. We argue that
\[
g(\mathcal{M}x) = g(\mathcal{M}x)\mathds{1}_{(\delta_x>0)} + g(\mathcal{M}x)\mathds{1}_{(\delta_x<0)} + g(\mathcal{M}x)\mathds{1}_{(\delta_x=0)}
\]
and analyze each term separately.
% Edeltax+ lemma

\textbf{Positive $\delta_x$.} We have
\[E[g(\mathcal{M}x)\mathds{1}_{(\delta_x>0)}]=g(x)\sum\limits_{y=1}^d\Pr(\delta_x=y)\tau^y.\]
By Lemma~\ref{lem:mutlaw},
\begin{align*}
    \sum\limits_{y=1}^d\Pr(\delta_x=y)\tau^y&\leq\sum\limits_{y=1}^d\dbinom{d}{y}\left(\frac{1}{n}\right)^{y}\tau^y\\
    &=\left(1+\frac{\tau}{n}\right)^{d}-1\leq\exp\left(\frac{d\tau}{n}\right)-1.
\end{align*}
As $f(x)\geq f_0\geq\alpha n$ we have $d\leq(1-\alpha)n$. By the definition of $\alpha$ (Definition~\ref{potentialG}),
\[E[g(\mathcal{M}x)\mathds{1}_{(\delta_x>0)}]\leq \left(\exp\left(\frac{d\tau}{n}\right)-1\right)g(x)\leq \frac{g(x)}{\tau}.\]
% neglected lemma
\textbf{Negative $\delta_x$.} We have
\begin{align*}
   E[g(\mathcal{M}x)\mathds{1}_{(\delta_x<0)}]&=g(x)\sum\limits_{y=f_0-f(x)}^{-1}\tau^y\Pr(\delta_x=y)\\
   &\leq\frac{g(x)}{\tau}\sum\limits_{y<0}\Pr(\delta_x=y)
   \leq\frac{g(x)}{\tau}.
\end{align*}
% keep fitness lemma
\textbf{Zero $\delta_x$.} By Lemma~\ref{lem:Pdelta_x=0} we have
\[E[g(\mathcal{M}x)\mathds{1}_{(\delta_x=0)}]=g(x)\sum\limits_{k=0}^d\dbinom{d}{k}\dbinom{n-d}{k}\left(\frac{1}{n}\right)^{2k}\left(1-\frac{1}{n}\right)^{n-2k}\]
\[=\left(1-\frac{1}{n}\right)^ng(x)+g(x)\sum\limits_{k=1}^d\dbinom{d}{k}\dbinom{n-d}{k}\left(\frac{1}{n}\right)^{2k}\left(1-\frac{1}{n}\right)^{n-2k}.\]
Hence
\[E[g(\mathcal{M}x)\mathds{1}_{(\delta_x=0)}]-\left(1-\frac{1}{n}\right)^ng(x)\leq g(x)\sum\limits_{k=1}^d\frac{d^k}{k!}\frac{(n-d)^k}{k!}\left(\frac{1}{n}\right)^{2k}.\]
Now, we have $\alpha>\frac 34$, thus $f(x)\geq f_0>\frac{n}{2}$ and $d=n-f(x)\leq n-f_0<\frac{n}{2}$. Since the polynomial $X(n-X)$ is monotonically increasing on $[0,\frac{n}{2}]$, we have $d(n-d)\leq (n-f_0)f_0$.
Thus $\frac{d(n-d)}{n^2}\leq\frac{n-f_0}{n}\leq 1-\alpha$, which leads to
\[E[g(\mathcal{M}x)\mathds{1}_{(\delta_x=0)}]-\left(1-\frac{1}{n}\right)^ng(x)\leq(1-\alpha)g(x)\sum\limits_{k=1}^\infty\frac{1}{k!^2}.\]
Let $b=\sum\limits_{k=1}^\infty\frac{1}{k!^2}$. By the analysis of the function $\tau\mapsto\ln\left(1+\frac{1}{\tau}\right)$, we have $(1-\alpha)b\leq \frac{2}{\tau}$ when $\tau>1$. By Lemma~\ref{logconcave}, $\left(1-\frac{1}{n}\right)^n\leq\frac{1}{e}$, thus
\[E[g(\mathcal{M}x)\mathds{1}_{(\delta_x=0)}]\leq \left(\frac{1}{e}+\frac{2}{\tau}\right)g(x).\]
Finally, we obtain
\begin{align*}
    E[g(\mathcal{M}x)]&=E\left[g(\mathcal{M}x)\mathds{1}_{(\delta_x>0)}\right]+ E\left[g(\mathcal{M}x)\mathds{1}_{(\delta_x<0)}\right]\\& + E\left[g(\mathcal{M}x)\mathds{1}_{(\delta_x=0)}\right]\leq \frac{1}{e}(1+\varepsilon)g(x).
\end{align*}
\end{proof}

\begin{proof}[Proof of Lemma~\ref{betterNeglected}]
Let $x$ be an individual such that $f(x)<f_0$. Then $\mathcal{M}x$ has potential zero unless it gains at least the difference between $f(x)$ and $f_0$, which is at least $1$. Using the law of total probability and Lemma~\ref{lem:mutlaw}, we have
\begin{align*}
    E[g(\mathcal{M}x)]&=\sum\limits_{k=f_0-f(x)}^{n-f(x)}\Pr(\delta_x=k)\tau^{k-f_0+f(x)}\\
    &\leq \frac{1}{\tau^{f_0-f(x)}}\sum\limits_{k=f_0-f(x)}^{n-f(x)}\dbinom{n-f(x)}{k}\left(\frac{\tau}{n}\right)^k\\
    &\leq \frac{\tau^{n-f_0}}{\tau^{n-f(x)}}\left(1+\frac{\tau}{n}\right)^{n-f(x)}=\tau^{n-f_0}\left(\frac{1+\frac{\tau}{n}}{\tau}\right)^{n-f(x)}.
\end{align*}
As $\tau>2$, we have $1+\frac{\tau}{n}<\tau$ for all $n$, hence this formula is monotonically increasing when $f(x)$ increases. Therefore, as $f(x)<f_0$, we have
\[E[g(\mathcal{M}x)]\leq\left(1+\frac{\tau}{n}\right)^{n-f_0}\leq\left(1+\frac{\tau}{n}\right)^{(1-\alpha)n}.\]
Due to the definition of $\alpha$ (Definition~\ref{potentialG}), by Lemma~\ref{logconcave} we deduce
\[E[g(\mathcal{M}x)]\leq 1+\frac{1}{\tau}\leq 2.\]
\end{proof}

\begin{proof}[Proof of Lemma~\ref{Ez0}]
If $U$ is the uniform random variable over $\E$, then
\[E[Z_0] \geq \tau^{n-f_0}-E[g(P_0)] = \tau^{n-f_0}-\mu E[g(U)].\]
Now, for all $x\in\E$, if $f(x)\leq \frac{3}{4}n$, then $f(x)<f_0$ and $g(x)=0$. Since $f(U)\sim\Bin(n,\frac{1}{2})$, with a Chernoff bound we have
\[\Pr\left(f(U)\geq \frac{3}{4}n\right)\leq \exp\left(-\frac{n}{12}\right).\]
Consequently $E[g(U)]\leq \tau^{n-f_0}\exp\left(-\frac{n}{12}\right)$, thus
\[E[Z_0] \geq \tau^{n-f_0}\left[1-\mu\exp\left(-\frac{n}{12}\right)\right].\]
Assuming that $\mu$ is sub-exponential, if $n$ is large enough then
\[E[Z_0]\geq \frac{1}{2}\tau^{n-f_0}.\]
\end{proof}

\begin{proof}[Proof of Theorem~\ref{boundedTopLevel}]
Assume that $X_t=s$ for some $s\in\N$. Let $(p_n)_\N$ be the sequence from Lemma~\ref{lem:PrifN_l}. By Lemma~\ref{lem:PrifN_l}, there exists $Y_1,\cdots, Y_\lambda\sim\Ber(p_n)$ i.i.d. such that $X_{t+1}\leq\sum\limits_{k=1}^\lambda Y_k:=B$. As the function $h:x\mapsto x(\ln\mu-\ln x+2)$ is monotonically increasing over $[0,\lambda]$, we have
\begin{align*}
    E[h(P_t)-h(P_{t+1})&\mid h(P_t)=h(s)]\geq E[h(s)-h(B)]\\
    &=(\ln\mu+2)E[s-B]+E[B\ln B-s\ln s].
\end{align*}
By Lemma~\ref{lem:PrifN_l} we have $E[B]=\lambda p_n\leq s$, thus
\[E[h(P_t)-h(P_{t+1})\mid h(P_t)=h(s)]\geq E[B\ln B-s\ln s].\]

Let $S_{min}$ be the constant from Theorem~\ref{log1plusBin}.

By Lemma~\ref{lem:PrifN_l} we have $\lambda p_n\underset{n\to+\infty}{\longrightarrow}s$ uniformly. Therefore if $n$ is large enough and if $s\geq S_{min}+1$, we have $(\lambda-1) p_n\geq S_{min}$. We use Theorem~\ref{log1plusBin} on $\sum\limits_{i=1}^{\lambda-1}Y_i\sim\Bin(\lambda-1,p_n)$ and deduce
\begin{align*}
    E[B\ln B] &= \sum\limits_{k=1}^\lambda E[Y_k\ln B]
    =p_n\sum\limits_{k=1}^\lambda E\left[\ln\left(1+\sum\limits_{i=1}^{\lambda-1}Y_i\right)\right]\\
    &\geq\lambda p_n\left(\ln(1+(\lambda-1)p_n)-\frac{11}{12}\frac{1-p_n}{(\lambda-1)p_n}\right).
\end{align*}
Let $\varepsilon=\frac{e-2}{36e}$ and let $S=\max\left\{S_{min}+1,\ceil{\frac{1}{\varepsilon}}+1\right\}$ so that, if $s\geq S$ and if $n$ is large enough, $\frac{(1-p_n)^2}{\lambda p_n}\leq \varepsilon$. In addition, by the uniform convergence, if $n$ is large enough, $s\ln s\leq \lambda p_n\ln(\lambda p_n)+\varepsilon$ and we have
\[E[B\ln B-s\ln s]\]
\begin{align*}
    &\geq \lambda p_n\left(\ln\left(1+\frac{1-p_n}{\lambda p_n}\right)-\frac{11(1-p_n)}{12\lambda p_n}\right)-\varepsilon\\
    &\geq \frac{1}{12}(1-p_n)-\frac{1}{2}\frac{(1-p_n)^2}{\lambda p_n}-\varepsilon\geq \frac{1-p_n}{12}-\tfrac{3}{2}\varepsilon.
\end{align*}
Now recall that $s\leq\mu$, so by the conditions (\ref{eq:lambdaleqemu}) we have $p_n\leq \frac{s}{\lambda}\leq\frac{\mu}{\lambda}\underset{n\to+\infty}{\longrightarrow}\frac{1}{e}$, therefore, if $n$ is large enough we have $1-p_n\geq1-\frac{2}{e}$. Consequently, if $s\geq S$, due to the choice of $\varepsilon$ we have
\[E[h(P_t)-h(P_{t+1})\mid P_t] \geq \frac{e-2}{24e}.\]
\end{proof}

\begin{proof}[Proof of Lemma~\ref{lem:large-mu-phase-1}]
% \da{Now only idea of the proof. Needs to be finished.}

% Given that the current level is $f$, after one generation we have at least $\frac{2\lambda}{3e\mu}\cdot \frac{\mu}{2} - O(\sqrt{\mu}) \ge \frac{\mu}{3}(1 - o(1))$ individuals on level $f + 1$ and at least
% $\frac{\lambda \mu}{2e\mu} - O(\sqrt{\mu}) = \frac{\mu}{2}(1 - o(1))$ copies of the individuals in level $f$. After one more generations the individuals from level $f + 1$ create about $\frac \mu 3$ copies of themselves and get the same amount of reinforcement from the lower level w.h.p. The probability to lose the level during these two generations is way too small.
%
% So we gain a level in two generations w.h.p. It gives us no more than $2n/3$ generations before the current level is at least $n/3$.

Let the current level be $f \le \frac{n}{3}$ at generation $t = 0$. Then we have $X_0 \ge \frac{\mu}{2}$.

We have $Y_1 \succeq \min\left\{\mu, \Bin(\lambda, \frac{2X_0}{3e\mu}) \right\}$, since we create an offspring with fitness at least $f + 1$ when we select an offspring with witness $f$ (with probability $\frac{X_0}{\mu}$) and flip only one wrong bit in it (with probability at least $\frac{2}{3e}$).
By Chernoff bounds we have the probability that $Y_1 < \frac{\mu}{3} (1 - \mu^{-\frac 13})$ is at most $\exp(-\frac{\mu^{\frac 13}}{6})$.

At the same time we have $X_1 \succeq \min\{\mu, \Bin(\lambda, \frac{X_0}{e\mu}) \}$ and by Chernoff bounds we have the probability that $X_1 < X_0(1 - \mu^{-\frac 13})$ is at most $\exp(-\frac{\mu^{\frac 13}}{4})$.

Given that $Y_1 \ge \frac{\mu}{3} (1 - \mu^{-\frac 13})$ and $X_1 \ge \frac{\mu}{2} (1 - \mu^{-\frac 13})$ we have $Y_2 \succeq \min\{\mu, \Bin(\lambda, \frac{2X_1}{3e\mu} + \frac{Y_1}{e\mu}) \}$. Therefore, by Chernoff bounds we have

\begin{align*}
  \Pr\left[Y_2 < \frac{\mu}{2}\right] &\le \Pr\left[Y_2 < \left(Y_1 + \frac{2X_1}{3}\right)(1 - \mu^{-\frac 13})\right] \\
                                      &\le \exp\left(-\frac{\frac{2\mu}{3}(1 - \mu^{-\frac 13})}{2\mu^{\frac 23}} \right) \le \exp\left(-\frac{\mu^{\frac 13}}{3} \right),
\end{align*}
if $\mu \ge 3$. By union bound we have the probability that the algorithm does not gain a level in two iterations is at most $\exp(-\frac{\mu^{\frac 13}}{6}) + \exp(-\frac{\mu^{\frac 13}}{4}) + \exp(-\frac{\mu^{\frac 13}}{3}) \le 3 \exp(-\frac{\mu^{\frac 13}}{6})$.

The probability that the algorithm reaches current level $\frac{n}{3}$ in not more than $\frac{2n}{3}$ iterations is at least $(1 - 3\exp(-\frac{\mu^{\frac 13}}{6}))^{2\frac n3} = 1 - o(1)$. If the algorithm does not reach current level $\frac{n}{3}$ in this number of iterations, in the worst case its current level is zero, and it starts another attempt. The probability that the algorithm needs another attempt is $o(1)$, so if $n$ is large enough the expected number of such attempts is at most $2$. This results that the expected number of generations of the first phase is at most $\frac{4n}{3} = O(n)$.

\end{proof}

\begin{proof}[Proof of Lemma~\ref{lem:large-mu-stay-forever}]
We split the runtime of the algorithm while its current level is $f$ into cycles. Each cycle can be either \emph{successful}, \emph{unsuccessful} or \emph{totally unsuccessful}. After a \emph{successful} cycle the algorithm has at least $\frac{\mu}{2}$ individuals of fitness exactly $f$ in the population. After an \emph{unsuccessful} cycle that started with $m$ individuals of fitness $f$ in the population, there are at least $m - 2\Delta_\mu$ individuals of fitness $f$ in the population, where $\Delta_\mu \coloneqq n^{\frac 23}h^{\frac 34}(n)$.
An uninterrupted series of $\frac{\mu}{8\Delta_\mu} = \frac{h^{\frac 14}(n)}{8}$ unsuccessful cycles results into a level loss.
A \emph{totally unsuccessful} cycle is a cycle that is neither successful, nor unsuccessful. After a totally unsuccessful cycle we pessimistically assume that the algorithm loses a level.
We aim to show that the event of a level loss is not likely to happen for long enough.

Each cycle is split into two phases. Consider some cycle that starts at generation $\tau_0$. To shorten the notation assume that $\tau_0 = 0$, however it does not mean that we regard only the first cycle. The first phase of the cycle terminates after $\tau_1$ generations, that is, at the first generation such that either $X_{\tau_1} < X_0 - \Delta_\mu$ or $Y_{\tau_1} \ge \mu_0$, where $\mu_0 \coloneqq n^{\frac 13}h(n)$.
If at the end of the first phase we have $X_{\tau_1} < X_0 - \Delta_\mu$ then the cycle is terminated and we consider it as either unsuccessful or totally unsuccessful if $X_{\tau_1} < X_0 - 2\Delta_\mu$. Otherwise, the cycle enters the second phase, which starts with at least $\mu_0$ individuals of fitness greater than $f$.

The second phase (and the cycle as well) ends after $\tau_2$ more generations, where $\tau_2$ is the first integer such that $X_{\tau_1 + \tau_2} \ge \frac{\mu}{2}$ or $X_{\tau_1 + \tau_2} < X_{\tau_1}$ or $Y_{\tau_1 + \tau_2} < \frac{\mu_0}{2}$.
If $X_{\tau_1 + \tau_2} \ge \frac{\mu}{2}$, then the cycle is successful, otherwise it is either unsuccessful or totally unsuccessful if $X_{\tau_1 + \tau_2} < X_{\tau_1}$.
% \da{Maybe, try to draw an illustration? However, we may be out of space.}\qy{Yeah, unfortunately. However if you can squeeze your written explanation and draw an illustration instead for the same amount of space it would be helpful. But we may be out of time!}\da{Now I am pretty sure that we move this part to the appendix, so no need to worry about the space. I will draw an illustration if we have some more time on Thursday.}

% We pessimistically assume that the algorithm loses a level as soon as the number of individuals in the population falls below $\frac{\mu}{4}$. This can happen only in two cases, namely (i) there was a totally unsuccessful cycle or (ii) there was an uninterrupted series of $\frac{\mu}{8\Delta_\mu} = h^{1/4}(n)$ unsuccessful cycles.
%
% Now we show that any cycle is successful with high probability. The cycle is successful, only if after the first phase we have $X_{\tau_1} \ge X_0 - \Delta_\mu$ and $Y_{\tau_1} \ge \mu_0$ and after the second phase we have $X_{\tau_1 + \tau_2} \ge \frac{\mu}{2}$.

Now we consider each phase in details to show that the probability that the cycle is successful is high.

\textbf{The First Phase.}
% \qy{note that if the cycle is totally unsuccessful, then it is unsuccessful. You should say that once and skip the neither nor in the future.}\da{We use the definition that unsuccessful cycle does not decrease $X_t$ by more than $\Delta_\mu$. So totally unsuccessful cycle is not just unsccessful. And it is easier to say it this way here that to alter the definition.}\qy{You should refer to the first phase of the first cycle instead of your current formulation. Because you use unsuccessful for both cycles and phases it gets confusing.}
The probability that the cycle is neither unsuccessful nor totally unsuccessful after the first phase is at least the probability that for $\tau_1^* \coloneqq n^{\frac 23}h^{\frac 14}(n)$ the number $Y_t$ of individuals on the upper levels reaches $\mu_0$ in less than $\tau_1^*$ generations while the number $X_t$ of individuals on the current level $X_t$ does not go below $X_0 - \Delta_\mu$%\qy{I still do not get it. I would say $X_0 - \Delta_\mu$ instead.}
until generation $\tau_1^*$.
% \qy{???}\da{See the new text above. Is it clearer? $\tau_1^*$ is some magically chosen number.}
By Lemma~\ref{lem:variation} we have the probability that $X_t \le X_0 - \Delta_\mu$ for some $t \le \tau_1^*$ is at most $\frac{\tau_1^* X_0}{\Delta_\mu^2} \le \frac{1}{h^{\frac 14}(n)}$.
%\qy{Don't forget to change / add the statement of the Lemma and to change this part accordingly} we have $\Var[X_{\tau_1^*}] \le \tau_1^* X_0$ and thus by Chebyshev's inequality the probability that $X_{\tau_1^*}$ is too small is %TODO

% \begin{align*}
%   \Pr[X_{\tau_1^*} < X_0 - \Delta_\mu] &= \Pr[X_{\tau_1^*} < E[X_{\tau_1^*}] - \Delta_\mu] \\
%                                      &\le \frac{\Var[X_{\tau_1^*}]}{\Delta_\mu^2} \le \frac{\tau_1^* X_0}{\Delta_\mu^2} \le \frac{1}{h^{\frac 14}(n)}.
% \end{align*}

To estimate the expected runtime before $Y_t$ becomes at least $\mu_0$ we apply Lemma~\ref{lem:dirty-trick} to $Y_t$. For this reason we note that $Y_{t + 1} \succeq \min\{\mu, \Bin(\lambda, \frac{Y_t}{e\mu} + \frac{X_t}{e n \mu})\}$,
% \qy{$Y$ is a number of individuals so it cannot be over $\mu$. You should either use the $\min\{-,\mu\}$ thing or say that you study the number of individuals before the selection I think. Then you'll have a stochastic domination so I guess the notation $\preceq$ should be used instead of $\sim$. Although it is a bit heavy, it is false otherwise...} \da{OK, I'll try to adjust the drift trick first.}
since we can obtain individuals in the upper levels either by copying one of $Y_t$ individuals with probability at least $\frac{1}{e}$ or by creating a superior offspring from one of $X_t$ individuals with probability at least $\frac{(n - f)}{en} \ge \frac{1}{en}$.
Since during the cycle we have $X_t$ at least $\frac{\mu}{4}$, we have $\Delta_{\min} \ge \frac{\mu}{4n}$. By taking $X' = \mu_0$, that is, at least $\max\{48, 18\ln\frac{\lambda}{\Delta_{\min}}$ if $n$ is large enough, and applying Lemma~\ref{lem:dirty-trick} we obtain that the expected runtime before $Y_t$ exceeds $\mu_0$ for the first time is at most $\max\{\frac{4\mu_0 n}{\mu}, 24\} = 4n^{\frac 23}$, if $n$ is large enough.
% As long as $f < n - \frac{1}{6}n^{2/3}$ this runtime is at most 24, otherwise, it is at most $\frac{4n^{2/3}}{n - f}$
By Markov's inequality the probability that this time exceeds $\tau_1^*$ is at most
\[
\frac{4n^{\frac 23}}{\tau_1^*} = \frac{4}{h^{\frac 14}(n)}.
\]
% \qy{I did not check this part, since it relies on the dirty trick Lemma.}
So as an intermediate result we have the probability that the cycle is neither unsuccessful, nor totally unsuccessful after the first phase is at least $1 -  \frac{1}{h^{\frac 14}(n)} - \frac{4}{h^{\frac 14}(n)} = 1 - \frac{5}{h^{\frac 14}(n)}$.

\textbf{The Second Phase.} Proceeding to the second phase of the cycle, we notice that it starts with $X_{\tau_1} \ge \frac{\mu}{4}$ and $Y_{\tau_1} \ge \mu_0$ and the cycle is successful if after the second phase we have $X_{\tau_1 + \tau_2} \ge \frac{\mu}{2}$.

 %\qy{same as above, I would say that if the second phase succeeds, then... However, I still do not understand what's happening here!}
% \da{I hope that after adding the details to the first phase I can leave it as it is}

The probability that the cycle succeeds after the second phase is at least the probability that for $\tau_2^* := n^{1/3}h^{1/3}(n)$ we have $Y_t \ge \frac{\mu_0}{2}$ for all $t \in [\tau_1..\tau_1 + \tau_2^*]$ and there exists some $t \le \tau_2^*$ such that $X_{\tau_1 + t} \ge \frac{\mu}{2}$ and $X_t$ does not decrease for $t \in [\tau_1, \tau_1 + \tau_2^*]$.

By  Lemma~\ref{lem:variation} the probability that $Y_t < \frac{\mu_0}{2}$ for some $t < \tau_1 + \tau_2^*$ is at most  $\frac{4\tau_2^* Y_{\tau_1}}{\mu_0^2} \le \frac{4}{h^{1/2}(n)}$.
% \begin{align*}
%   \Pr\left[Y_{\tau_1 + \tau_2^*} < Y_{\tau_1} - \frac{\mu_0}{2}\right] &= \Pr\left[Y_{\tau_1 + \tau_2^*} < E[Y_{\tau_1 + \tau_2^*}] - \frac{\mu_0}{2}\right] \\
%                                                                      &\le \frac{\Var[Y_{\tau_1 + \tau_2^*}]}{(\mu_0/2)^2} \le \frac{4\tau_2^* Y_{\tau_1}}{\mu_0^2} \le \frac{4}{h^{1/2}(n)}.
% \end{align*}

By the application of Lemma~\ref{lem:dirty-trick} to $X_{t + 1} \sim \Bin(\lambda, \frac{X_t}{e\mu} + \frac{fY_t}{e\mu})$ we have the expected runtime before $X_t$ becomes more than $\frac{\mu}{2}$ is at most $\frac{4\mu/2 - 2\mu/4}{\frac{\mu_0}{6}} = 9n^{\frac 13}$.
By Markov's inequality we have that the probability that $X_t$ does not become greater than $\frac{\mu}{2}$ in less than $\tau_2^*$ generations is at most $\frac{9n^{\frac 13}}{\tau_2^*} = \frac{9}{h^{\frac 12}(n)}$.

Finally, by Chernoff bounds the probability that $X_t$ decreases in one generation during the second phase is at most $\exp(-\frac{Y_t^2}{9E[X_t]}) \le \exp(-\frac{\mu_0^2}{36\mu}) \le \exp(-\frac{h(n)}{36})$. The probability that $X_t$ does not decrease during $\tau_2^*$ generations is at least $(1 - \exp(-\frac{h(n)}{36}))^{n^{\frac 13}h^{\frac 12}(n)}$.
% \da{the first place, where we use $h(n)$:}
Since we have $h(n) \ge \ln^4(n)$, if $n$ is large enough this probability is at least $1 - \frac{1}{h^{\frac 12}(n)}$.

Summing up, the probability that the cycle is not successful after the second phase is at most $\frac{4}{h^{\frac 12}(n)} + \frac{9}{h^{\frac 12}(n)} + \frac{1}{h^{\frac 12}(n)} + \le \frac{15}{h^{\frac 12}(n)}$. If we also take into account the probability not to be unsuccessful after the first phase, we estimate the probability of a successful cycle $p_s$ as

\begin{align*}
p_s \ge 1 - \frac{5}{h^{\frac 14}(n)} - \frac{15}{h^{\frac 12}(n)} \ge 1 - \frac{6}{h^{\frac 14}(n)},
\end{align*}
if $n$ is large enough.

\textbf{Losing a Level.} The cycle is totally unsuccessful in two cases. The first case is when in the end of the first phase we have $X_{\tau_1} < X_0 - 2\Delta_\mu$. Notice that for this to happen we need $X_{\tau_1 - 1} - X_{\tau_1} > \Delta_\mu$, and by Chernoff bounds the probability of this is at most $\exp(-n^{\frac 23}h^{\frac 12}(n))$. The second case is when in the end of the second phase we have $X_{\tau_1 + \tau_2} < X_{\tau_1} - \Delta_\mu$.
The probability of this event is at most the probability that at the last generation of the second phase $X_t$ decreased, that is, at most $\exp(-\frac{h(n)}{9})$. Hence, the probability $p_{tu}$ of a totally unsuccessful cycle is at most

\begin{align*}
p_{tu} &\le \exp(-n^{\frac 23}h^{\frac 12}(n)) + \exp(-\frac{h(n)}{9})\\& \le \exp(-h^{\frac 12}(n)) \le \exp(-\ln^2(n)) \le \frac{1}{n^3},
\end{align*}
if $n > e^3$, since we assume that $h(n) \ge \ln^4(n)$.

Except an unsuccessful cycle we have only one possible way to lose a level, that is a series of $\frac{\mu}{8\Delta_\mu}$ unsuccessful cycles in a row.
The probability $p_{us}$ of such series is at most
\begin{align*}
 p_{us} \le (1 - p_s)^\frac{\mu}{8\Delta_\mu} \le \left(\frac{6}{h^{\frac 14}(n)}\right)^{\frac{h^{\frac 14}(n)}{8}} \le \frac{1}{n^3},
\end{align*}
if $n$ is large enough.

Since each cycle takes at least one generation, the probability not to lose a level after $t$ generations is at most $(1 - p_{tu} - p_{us})^t \ge (1 - \frac{2}{n^3})^t$.

\end{proof}

\end{document}